\documentclass{article}

\usepackage[main, final]{neurips_2026}

\usepackage[utf8]{inputenc} % allow utf-8 input
\usepackage[T1]{fontenc}    % use 8-bit T1 fonts
\usepackage{hyperref}       % hyperlinks
\usepackage{url}            % simple URL typesetting
\usepackage{booktabs}       % professional-quality tables
\usepackage{amsfonts}       % blackboard math symbols
\usepackage{nicefrac}       % compact symbols for 1/2, etc.
\usepackage{microtype}      % microtypography
\usepackage{xcolor}         % colors

\usepackage[utf8]{inputenc}
\usepackage[T1]{fontenc}
\usepackage{microtype}
\usepackage{amsmath, amssymb, amsthm}
\usepackage{cleveref}
\usepackage{xcolor}
\usepackage{graphicx}
\usepackage{booktabs}
\usepackage{multirow}
\usepackage{array}
\usepackage{algorithm, algpseudocode}
\usepackage{enumitem}
\usepackage{pgfplots}
\pgfplotsset{
  compat=1.18,
  myplot/.style={
    tick label style={font=\scriptsize},
    label style={font=\scriptsize},
    title style={font=\scriptsize},
    legend style={
      font=\scriptsize,
      draw=none,
      fill=none,
      /tikz/every even column/.append style={column sep=0.6em},
    },
    grid=both,
    grid style={dotted, gray!30},
  }
}
\usepgfplotslibrary{groupplots}
\usepackage{amsmath} 
\usepackage{tikz}
\usetikzlibrary{arrows.meta, positioning, fit, backgrounds, calc, shapes.geometric, decorations.pathreplacing}

\usepackage{tikz}
\usetikzlibrary{
  arrows.meta,
  positioning,
  fit,
  backgrounds,
  calc,
  decorations.pathreplacing
}
\usepackage{graphicx} % for \resizebox

\usepackage{booktabs}
\usepackage{tabularx}
\usepackage{array}
\usepackage[table]{xcolor}

\definecolor{TableHeader}{RGB}{245,247,250}
\definecolor{TableGroup}{RGB}{250,250,252}
\definecolor{BestRow}{RGB}{237,247,237}
\definecolor{SecondRow}{RGB}{239,244,252}

\hypersetup{
colorlinks=true
}
\newtheorem{proposition}{Proposition}
\newtheorem{theorem}{Theorem}
\newtheorem{definition}{Definition}
\newtheorem{assumption}{Assumption}
\newtheorem{corollary}{Corollary}

\title{Behavior Pack Optimization \\for Video MLLM Post-Training}

\author{%
  \textbf{Zhaolu Kang}$^{1,2*}$, \textbf{Shiyu Liu}$^{1,*}$, \\\textbf{Tailong Luo}, \textbf{Wei Zhang}, \textbf{Yingjie He}$^{1}$, 
  \textbf{Lei Wei}$^{1}$, \textbf{Guansu Wang}$^{3}$, \textbf{Liang He}$^{4}$, \textbf{Siheng Wang}, \\\textbf{Guangyuan Dong}, 
  \textbf{Jiaqi Su}, \textbf{Shuang Chen}, \textbf{Haoyu Ji}, \textbf{Qishi Zhan}, \textbf{Kaiyue Zhou}$^{2,\dagger}$ \\
  $^{1}$Peking University \quad $^{2}$Chengdu Minto Tech \\
  $^{3}$The University of Melbourne \quad $^{4}$Stanford University \\
  \texttt{zlkang25@stu.pku.edu.cn}
}

\begin{document}

\maketitle

\renewcommand{\thefootnote}{\fnsymbol{footnote}}
\footnotetext[1]{Equal Contribution. $^{\dagger}$Corresponding Author.}
\renewcommand{\thefootnote}{\arabic{footnote}}

% ===========================================================
%  §0 Abstract  — 180–200 words target
%  Status: draft v0.5 (prose macros inlined; method name = BPO)
% ===========================================================
\begin{abstract}
Video multimodal large language models (MLLMs) keep climbing video question
answering benchmarks, yet shuffling the frames, masking the segment that
supports the answer, or occluding the target object barely changes their
predictions. The accuracy rests on appearance and language priors, not on the
temporal evidence the question asks for. We trace this to the unit of
post-training: rewards are computed on a single response to the original clip,
so the model is never asked to behave consistently across views. We propose
Behavior Pack Optimization (BPO), which replaces the single response with a
behavior pack of outputs across counterfactual views chosen by question type,
scored jointly. The pack reward asks for stability when the intervention is
irrelevant, sensitivity when key evidence is removed, and abstention when no
evidence remains. To keep this objective stable at small pack sizes, BPO uses
an anchor-relative advantage: the response on the original view serves as a
per-prompt reference instead of a group mean over mixed views. On TempCompass,
MVBench, and NExT-QA, BPO improves the macro accuracy of
Qwen2.5-VL-7B-Instruct by 4.7 pp, the temporal-hard subset by 7.8 pp, and
abstention F1 by 20.0 pp over a budget-matched vanilla GRPO baseline from the
same SFT checkpoint. The gains transfer to Video-MME, LongVideoBench, and to
LLaVA-Video-7B; ablations confirm they follow the view sets, not the rollout
count. We hope this pack-level perspective offers a useful starting point for
the video MLLM and multimodal post-training community as the field moves toward
evidence-grounded video reasoning.
\end{abstract}
% ===========================================================
%  §1 Introduction  — target ~1.0–1.2 page
%  Style: NeurIPS calm-register, modeled on DPO / CoT / Toolformer /
%  DeepSeekMath. NO formulas, NO em-dashes, NO \emph, NO own
%  experiment numbers in the problem section. P2 articulates the
%  problem in prior work; P3 introduces the method.
%  Status: draft v0.6
% ===========================================================
\section{Introduction}
\label{sec:intro}

Video multimodal large language models (MLLMs) are now common
backbones for video question answering.
% %, with strong performance on
% benchmarks such as MVBench~\cite{li2024mvbench},
% TempCompass~\cite{liu2024tempcompass}, and
% NExT-QA~\cite{xiao2021next}. 
Their post-training recipe, however, still often follows the text-only LLM pipeline: supervised fine-tuning on instruction-formatted question-answer pairs, optionally followed by preference optimization or group-relative reinforcement learning~\cite{rafailov2023direct, shao2024deepseekmath, guo2025deepseek}. Recent video RL methods add temporal rewards,
augmented negatives, or pairwise preferences over manipulated
clips~\cite{feng2025video, li2025temporal}. These signals are useful,
but the optimized object often remains local: one response on one
video, a same-input response group, or a comparison tied to one edited
clip. This design fits static answer matching, but it misses a defining
property of video QA: the answer should often change, or become
unsupported, when the visual evidence is changed or removed.

A local optimization unit creates a structural gap. When the loss
grades responses in isolation, the policy receives no direct signal
about whether its behavior changes appropriately across related
counterfactual views. Benchmarks such as MVBench and TempCompass stress
temporal order, direction, and evidence-dependent
reasoning~\cite{li2024mvbench, liu2024tempcompass}, while controlled
probes show that video MLLMs can preserve answers under frame
shuffling, object occlusion, or removal of answer-supporting segments.
Adding manipulated clips broadens the training distribution, but if
each response is still graded against a local target, the objective
does not specify how the model should behave across interventions. The
bottleneck is therefore not only what the model is shown, but also the
granularity at which behavior is optimized.

We propose \textbf{B}ehavior \textbf{P}ack \textbf{O}ptimization (BPO), a video RL post-training
framework that optimizes a behavior pack instead of an isolated
response. For each video-question pair, BPO constructs a compact pack
with the original clip and question-type-conditioned counterfactual
views: evidence-removal views for event, binding, and count questions,
and temporal interventions such as reversal or shuffling for order and
direction questions. The policy is rolled out on every view, and the
outputs are scored jointly by a verifier-based reward that combines
per-view correctness with three cross-view contracts:
stability under label-preserving perturbations, sensitivity to
answer-changing interventions, and honest abstention when no diagnostic
evidence remains. Thus, BPO is not simply data augmentation: it changes
the optimization unit from a single video-answer pair to a cross-view
behavior pack.

Within this objective, the advantage estimator also matters. Standard
GRPO centers rewards by statistics of the sampled group, but video
post-training often permits only small groups because of memory and
decoding cost. BPO instead uses the rollout on the original,
unintervened clip as a per-prompt anchor, yielding an anchor-relative
advantage across the pack. The anchor is graded against the
ground-truth answer on the original clip, so intervened-view rewards
alone cannot set the reference. We analyze this estimator in
Section~\ref{sec:method:anchor} and show that it can reduce variance
when rewards across the anchor and intervened views are correlated.

We evaluate BPO on TempCompass, MVBench, and NExT-QA~\cite{xiao2021next} using
Qwen2.5-VL-7B-Instruct as the backbone. All main comparisons use the
same SFT initialization and match the per-prompt rollout budget under
compute-controlled training. Against vanilla GRPO and budget-matched
re-implementations of Video-R1~\cite{feng2025video} and
TPO~\cite{li2025temporal}, BPO improves macro accuracy by~4.7\,pp,
temporal-hard accuracy by~7.8\,pp, and abstention F1 on
insufficient-evidence subsets by~20.0\,pp. The behavior indicators
move in the directions specified by the contracts; ablations separate
type-aware interventions from extra rollouts; and the gains transfer
to longer-video evaluation and a second backbone.

Our contributions are:
\begin{itemize}[leftmargin=*, noitemsep, topsep=0pt]
\item We identify optimization granularity, rather than data coverage
alone, as a bottleneck connecting shortcut behavior in video QA to the
post-training objective.
\item We introduce BPO, a pack-level objective that combines
question-type-conditioned view sets, cross-view behavior contracts,
and an anchor-relative advantage estimator.
\item Through budget-aligned comparisons, we show that BPO improves
accuracy, abstention, and contract-aligned behavior indicators, with
gains transferring to longer-video evaluation and a second backbone.
\end{itemize}

\begin{figure*}[!t]
\centering
\resizebox{\textwidth}{!}{%
\begin{tikzpicture}[
  x=1cm, y=1cm,
  font=\small,
  >=Latex,
  arr/.style={-{Latex[length=2mm]}, thin},
  darr/.style={-{Latex[length=2mm]}, dashed, thin},
  clip/.style={
    draw, rounded corners=2pt, minimum width=1.1cm, minimum height=0.58cm,
    inner sep=2pt, fill=gray!8
  },
  resp/.style={
    draw, rounded corners=2pt, minimum width=1.0cm, minimum height=0.52cm,
    inner sep=2pt, fill=gray!4
  },
  view/.style={
    draw, rounded corners=2pt, minimum width=1.25cm, minimum height=0.58cm,
    inner sep=2pt, fill=blue!7
  },
  anchorv/.style={
    draw, rounded corners=2pt, minimum width=1.6cm, minimum height=0.58cm,
    inner sep=2pt, fill=orange!10, line width=0.8pt
  },
  rewbox/.style={
    draw, rounded corners=2pt, minimum width=5.6cm, minimum height=0.62cm,
    inner sep=3pt, fill=green!8
  },
  ann/.style={font=\footnotesize\itshape, text=red!70!black}
]

% ================= (a) GRPO on the left =================
\node[font=\bfseries\large, align=center] (tA) at (0,2.55)
  {(a) GRPO: response-level, sampled-response baseline};

\node[clip] (gv) at (0,1.45) {$v_0$};

\node[resp] (g1) at (-1.80,0.25) {$y_1$};
\node[resp] (g2) at (-0.65,0.25) {$y_2$};
\node[font=\small] (gd) at (0.45,0.25) {$\cdots$};
\node[resp] (g3) at (1.55,0.25) {$y_K$};

\foreach \t in {g1,g2,g3} \draw[arr] (gv) -- (\t);
\draw[arr] (gv) -- (gd);

\draw[decorate, decoration={brace, amplitude=3pt, mirror}]
  (-2.20,-0.15) -- (2.00,-0.15)
  node[midway, below=4pt, font=\small] {sample mean $\bar r$};

\node[font=\small] at (0,-1.00) {advantage: $A_k = r_k - \bar r$};

% ================= vertical divider =================
\draw[gray!50, dashed] (4.575,-2.00) -- (4.575,2.85);

% ================= (b) BPO on the right =================
\node[font=\bfseries\large, align=center] (tB) at (9.15,2.55)
{(b) BPO : pack-level, anchor-relative baseline};

\node[clip] (bv) at (9.15,1.65) {video};

\node[anchorv] (v0) at (6.70,0.45) {$v_0$ (anchor)};
\node[view]    (vk) at (8.45,0.45) {key-mask};
\node[view]    (vr) at (10.15,0.45) {reverse};
\node[view]    (vs) at (11.85,0.45) {shuffle};

\foreach \t in {v0,vk,vr,vs} \draw[arr] (bv) -- (\t);

\node[resp] (y0) at (6.70,-0.70) {$y_{v_0}$};
\node[resp] (yk) at (8.45,-0.70) {$y_{\mathrm{km}}$};
\node[resp] (yr) at (10.15,-0.70) {$y_{\mathrm{rev}}$};
\node[resp] (ys) at (11.85,-0.70) {$y_{\mathrm{shf}}$};

\draw[arr] (v0) -- (y0);
\draw[arr] (vk) -- (yk);
\draw[arr] (vr) -- (yr);
\draw[arr] (vs) -- (ys);

\node[rewbox] (rc) at (9.15,-1.65)
  {cross-view reward $R^{(v)} = R_{\mathrm{acc}} + R_{\mathrm{contract}} - \mathrm{Pen}$};

\foreach \t in {y0,yk,yr,ys} \draw[arr] (\t) -- (rc);

% pack dashed frame (optimization unit)
\begin{pgfonlayer}{background}
  \node[
    draw=red!65!black,
    dashed,
    line width=0.8pt,
    rounded corners=2pt,
    fit=(v0)(vk)(vr)(vs)(y0)(yk)(yr)(ys),
    inner xsep=6pt,
    inner ysep=5pt
  ] (pack) {};
\end{pgfonlayer}

% annotations
\node[ann, anchor=west, align=left] at (12.75,0.00)
  {pack\\= optimization unit};
\draw[darr, red!65!black] (12.65,-0.02) -- ($(pack.east)+(0.03,0)$);

\draw[darr, orange!75!black, line width=0.8pt]
  (y0.north east) to[bend left=23]
  node[above, font=\scriptsize, text=orange!70!black]
  {$A_v = r_v - r_{v_0}$}
  (yk.north west);

\end{tikzpicture}%
}
\vspace{-1.7em}
\caption{(a) GRPO samples $K$ rollouts from the original clip $v_0$, centers rewards by the sampled-response mean $\bar r$, and optimizes each response separately.
(b) BPO samples one rollout per view from a question-type-conditioned set, scores the pack jointly with cross-view behavior contracts, and centers advantages relative to the $v_0$ anchor rollout.
% Red box: optimization unit; orange arrow: baseline reference.
}
\label{fig:teaser}
\vspace{-0.3em}
\end{figure*}
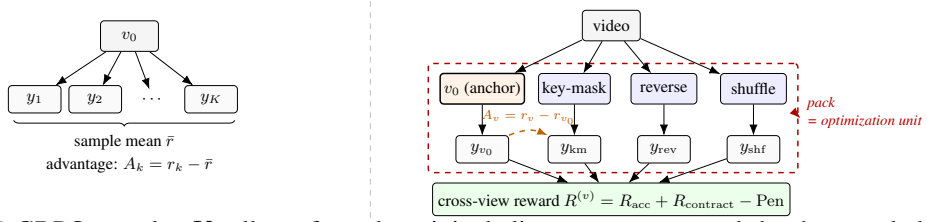
% ===========================================================
%  §2 Related Work  — target ~0.4 page, two compact paragraphs,
%  no per-method paragraph heads, no comparison table in body.
%  A full per-axis discussion and the comparison table live in
%  Appendix G.
% ===========================================================
\begingroup
\makeatletter
\renewcommand\paragraph{\@startsection{paragraph}{4}{\z@}%
  {0.35ex plus 0.1ex minus 0.1ex}% space before paragraph heading
  {-0.75em}% run-in space after heading
  {\normalfont\normalsize\bfseries}}
\makeatother

\section{Related Work}
\label{sec:related}

\paragraph{Faithful temporal grounding.}
Video MLLMs are commonly post-trained with supervised fine-tuning on
instruction-style video question-answer pairs~\cite{zhang2024llava,
wang2024internvideo2, bai2025qwen25vltechnicalreport}, and evaluated on
benchmarks that stress temporal order, multi-step events, and long-video
reasoning~\cite{li2024mvbench, liu2024tempcompass, xiao2021next,
fu2025video, wu2024longvideobench}. High benchmark scores, however, do
not guarantee faithful temporal grounding. A model may answer from
appearance cues or language priors, or preserve the same answer after the
supporting evidence is reordered, masked, or removed. Counterfactual
evaluation and evidence-sensitivity diagnostics therefore complement
static answer matching: a faithful model should remain stable under
label-preserving perturbations, but change or abstain when an
intervention changes or removes the evidence needed to
answer~\cite{li2025temporal, poppi2026countervid}.

\paragraph{Video RL and counterfactual supervision.}
Preference optimization and reinforcement learning offer a way to improve
behavior beyond supervised imitation. DPO learns from preference pairs
without an explicit reward model, while GRPO removes the value model and
estimates relative advantages from responses sampled for the same
prompt~\cite{rafailov2023direct, shao2024deepseekmath, guo2025deepseek}.
Recent video post-training methods introduce video-specific rewards and
training signals: Video-R1 uses shuffled-frame temporal rewards,
VideoChat-R1 applies GRPO-style RFT to spatio-temporal perception,
VideoRFT adds semantic-consistency rewards, DeepVideo-R1 studies
difficulty-aware GRPO variants, and SpaceR focuses on spatial
RLVR~\cite{feng2025video, li2025videochat, wang2025videorft,
park2025deepvideo, ouyang2025spacer}. Another line constructs
preferences from edited, temporally perturbed, or synthesized
counterfactual videos~\cite{li2025temporal, poppi2026countervid}. These
works broaden the supervision available to video post-training, but the
training signal is typically applied to an individual response, a
same-input response group, or a pairwise comparison tied to one
manipulated instance, rather than to the joint behavior across a
structured set of views.

\paragraph{Behavior-pack optimization.}
BPO changes this optimization granularity. It optimizes a compact
question-type-conditioned behavior pack in which outputs from the
original clip and diagnostic counterfactual views are scored jointly
under contracts for stability, sensitivity, and honest abstention. This
is orthogonal to many reward or preference designs: BPO keeps the
GRPO-style clipped surrogate, but changes the unit being optimized and
replaces the batch-normalized pack reference with an anchor-relative
reference from the original-clip rollout. It is also distinct from
reasoning-trace or representation-based approaches, which improve how
models structure intermediate reasoning~\cite{shao2024visual}; BPO
changes the unit over which video behavior is rewarded.

\endgroup

% ===========================================================
%  §3 Method
%  Polished version: formulation-driven, Video-R1-style,
%  experiment-preserving, and ambiguity-controlled.
% ===========================================================
\section{Method}
\label{sec:method}

We cast video RL post-training as cross-view behavior optimization. For each
video-question sample, Behavior Pack Optimization (BPO) builds the original
view and a small set of question-typed counterfactual views. The policy
produces one rollout per view; a verifier scores view-level correctness and
cross-view contracts relative to the original-view anchor.

Compared with vanilla GRPO, BPO changes both the optimization unit and the
baseline: it trains on behavior packs rather than isolated responses, and it
uses an anchor-relative advantage rather than a batch-normalized pack baseline. This makes
the training signal depend on whether answers change, stay stable, or abstain
under evidence interventions. \Cref{fig:bpo_overview} summarizes the pipeline;
the following subsections define its components.

\begin{figure}[t]
\centering
\includegraphics[width=\linewidth]{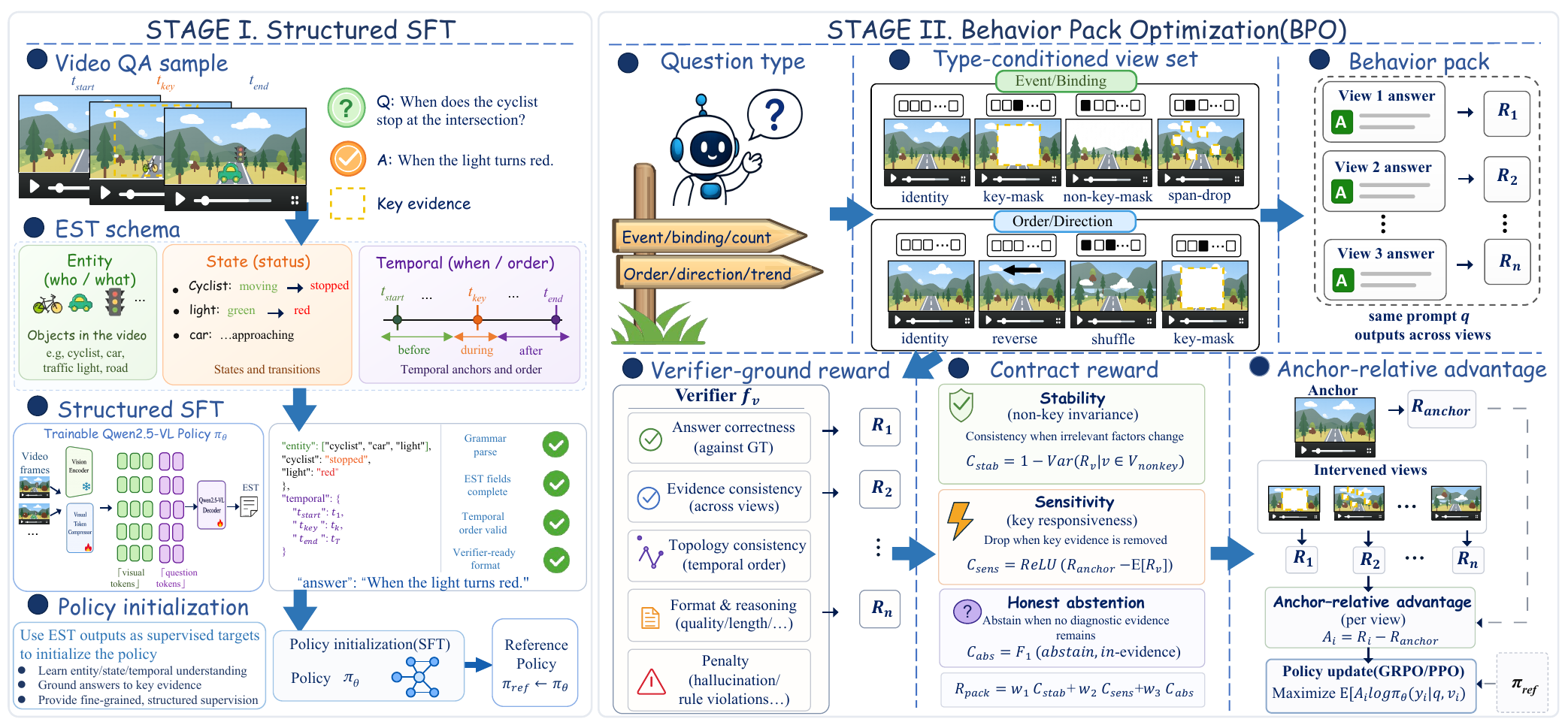}
\vspace{-1.2em}
\caption{Overview of Behavior Pack Optimization. Stage~A installs the
structured EST interface with SFT. Stage~B routes each question to a
question-typed view set, forms a behavior pack across counterfactual views,
scores view-specific and cross-view contract rewards, and updates the policy
with anchor-relative advantages.}
\label{fig:bpo_overview}
\end{figure}

% -----------------------------------------------------------
\subsection{Setup and view records}
\label{sec:method:setup}

A training sample is $(\mathbf{V},q,y^\star)$, optionally with structured
metadata such as key-evidence segments, target objects, and counterfactual
annotations. The policy $\pi_\theta$ outputs an answer and an Entity-State-Temporal
(EST) trace installed during Stage~A SFT (Appendix~\ref{app:est}); a deterministic
grammar parses the trace with success rate $\geq 0.99$.

We represent each view as a record
\begin{equation}
  \omega_v=(v,\mathbf{V}_v,q_v,y_v^\star,\xi_v),
  \label{eq:view-record}
\end{equation}
where $v$ denotes the intervention type, $\mathbf{V}_v$ is the transformed
video, $q_v$ is the corresponding question, $y_v^\star$ is the view-specific
target answer, and $\xi_v$ stores intervention metadata used by the verifier
and the contract reward. The identity view $v_0$ satisfies
\begin{equation}
  (\mathbf{V}_{v_0},q_{v_0},y_{v_0}^\star)
  =
  (\mathbf{V},q,y^\star).
  \label{eq:identity-view}
\end{equation}
For counterfactual views, $y_v^\star$ is inherited for label-preserving
views, derived from deterministic counterfactual rules when available,
or replaced by an insufficient-evidence target when the intervention
removes the required support. The full view pool is
\begin{equation}
  \mathcal{V}_{\mathrm{full}}
  =
  \{v_0,\text{key-mask},\text{non-key-mask},
  \text{span-drop},\text{reverse},\text{shuffle}\}.
  \label{eq:full-view-pool}
\end{equation}
View construction details are in Appendix~\ref{app:views}.

% -----------------------------------------------------------
\subsection{Behavior packs}
\label{sec:method:pack}

For a sample $(\mathbf{V},q)$ and a view set
$\mathcal{V}\subseteq\mathcal{V}_{\mathrm{full}}$ with
$v_0\in\mathcal{V}$, the policy produces a behavior pack
\begin{equation}
  \mathcal{B}(\mathbf{V},q;\mathcal{V})
  =
  \Bigl\{
  (v,y_v):
  v\in\mathcal{V},\
  \omega_v=(v,\mathbf{V}_v,q_v,y_v^\star,\xi_v),\
  y_v\sim\pi_\theta(\cdot\mid \mathbf{V}_v,q_v)
  \Bigr\}.
  \label{eq:pack}
\end{equation}
Unlike GRPO response groups, BPO samples across views of the same prompt, so a
pack exposes whether the policy preserves, changes, or abstains as evidence is
intervened on. This is different from treating counterfactual clips as ordinary
data augmentation: the views are decoded in the same optimization unit and are
scored jointly through contracts to the original view. As a result, the update
can reward a response that stays unchanged under a label-preserving edit while
penalizing the same unchanged answer after evidence has been removed.

\paragraph{Pack-level batch baseline.}
For the ablation baseline, the scalar pack reward is the mean per-view reward:
\begin{equation}
  R_{\mathrm{pack}}(\mathcal{B})
  =
  \frac{1}{|\mathcal{V}|}
  \sum_{(v,y_v)\in\mathcal{B}}
  R_{\mathrm{view}}(v,y_v;y_{v_0}).
  \label{eq:rpack}
\end{equation}
For a mini-batch of $G$ packs, the batch-normalized pack advantage is
\begin{equation}
  \hat{A}_{\mathrm{pack}}^{(i)}
  =
  \frac{
  R_{\mathrm{pack}}\bigl(\mathcal{B}^{(i)}\bigr)-\mu_G
  }{
  \sigma_G+\epsilon
  },
  \qquad
  \mu_G,\sigma_G \text{ computed over the } G \text{ packs in the mini-batch}.
  \label{eq:advpack}
\end{equation}
The clipped surrogate applies this shared pack-level advantage view-wise:
\begin{equation}
  \mathcal{L}_{\mathrm{group}}
  =
  -\mathbb{E}_{\mathcal{B}^{(i)}}\!\left[
    \frac{1}{|\mathcal{V}|}
    \sum_{(v,y_v)\in\mathcal{B}^{(i)}}
    \min\!\left(
      \rho_v\hat{A}_{\mathrm{pack}}^{(i)},
      \mathrm{clip}(\rho_v,1-\epsilon,1+\epsilon)
      \hat{A}_{\mathrm{pack}}^{(i)}
    \right)
  \right],
  \label{eq:loss-group}
\end{equation}
where
\begin{equation}
  \rho_v
  =
  \frac{
  \pi_\theta(y_v\mid \mathbf{V}_v,q_v)
  }{
  \pi_{\mathrm{old}}(y_v\mid \mathbf{V}_v,q_v)
  }
  \label{eq:importance-ratio}
\end{equation}
is the per-view importance ratio. With only the identity view and no
cross-view contracts, this reduces to a single-view clipped update. With
multiple views, the shared pack reward couples responses through the same
intervention set.

\paragraph{Small-pack estimation.}
Since video RL permits only small packs ($|\mathcal{V}|\leq 4$ and
$G\approx8$--$16$), batch centering in \Cref{eq:advpack} mixes prompt difficulty
with intervention behavior and can yield noisy variance estimates. The
normalization is over packs in the mini-batch, not over multiple rollouts of the
same prompt as in standard GRPO. BPO replaces this batch-level reference with a
prompt-level anchor.

% -----------------------------------------------------------
\subsection{Anchor-relative optimization}
\label{sec:method:anchor}

BPO uses the original-view rollout as a prompt-specific reference for all
intervened views, avoiding reward comparisons across unrelated prompts.

\begin{definition}[Anchor-relative advantage]
\label{def:anchor}
For pack
$\mathcal{B}^{(i)}=\{(v,y_v^{(i)})\}_{v\in\mathcal{V}}$ with anchor view
$v_0$, the anchor-relative advantage of view $v$ is
\begin{equation}
  \hat{A}_{\mathrm{anc}}^{(i,v)}
  =
  R_{\mathrm{view}}\bigl(v,y_v^{(i)};y_{v_0}^{(i)}\bigr)
  -
  R_{\mathrm{view}}\bigl(v_0,y_{v_0}^{(i)};y_{v_0}^{(i)}\bigr).
  \label{eq:advanchor}
\end{equation}
The anchor view has zero self-relative advantage,
$\hat{A}_{\mathrm{anc}}^{(i,v_0)}=0$.
\end{definition}

The BPO surrogate applies the clipped policy-gradient objective to the
non-anchor views:
\begin{equation}
  \mathcal{L}_{\mathrm{BPO}}
  =
  -\mathbb{E}_{i}\!\left[
    \frac{1}{|\mathcal{V}|}
    \sum_{(v,y_v)\in\mathcal{I}^{(i)}}
    \min\!\left(\rho_v A_v,c_v A_v\right)
  \right]
  +
  \beta_{\mathrm{KL}}\,
  \mathbb{E}_{v\in\mathcal{V}}
  \mathrm{KL}\!\left[p_{\theta,v}\Vert p_{\mathrm{ref},v}\right].
  \label{eq:loss-anchor}
\end{equation}
where $\mathcal{I}^{(i)}=\{(v,y_v)\in\mathcal{B}^{(i)}:v\neq v_0\}$,
$A_v=\hat{A}_{\mathrm{anc}}^{(i,v)}$,
$c_v=\mathrm{clip}(\rho_v,1-\epsilon,1+\epsilon)$,
$p_{\theta,v}=\pi_\theta(\cdot\mid \mathbf{V}_v,q_v)$, and
$p_{\mathrm{ref},v}=\pi_{\mathrm{ref}}(\cdot\mid \mathbf{V}_v,q_v)$.
The normalization remains pack-level; the anchor has zero relative advantage
and acts as the grounded control for all non-anchor updates, while the KL term
keeps Stage~B close to the Stage~A reference policy. We do not use the
anchor-relative term to further optimize the original-view answer itself:
otherwise a high- or low-reward anchor could move the reference it is supposed
to provide. The original view contributes through the KL term and through its
reward as the baseline used to form non-anchor advantages. We log its verifier
score in the pack diagnostics, but do not apply an additional RL
log-probability term to the anchor response.

\paragraph{Estimator stability and budget.}
Under the exchangeable reward model in Appendix~\ref{app:theory:variance}, the
anchor-relative update signal has lower unnormalized variance than the
batch-centered pack signal when
$\rho>\rho^\star(m)=(2m-1)/(3m-1)$. The measured correlations in
\Cref{sec:exp:anchor} exceed this threshold. We also match the total rollout
budget to vanilla GRPO by choosing group sizes with
$|\mathcal{V}|G_{\mathrm{BPO}}\approx K_{\mathrm{vanilla}}G_{\mathrm{vanilla}}$
within $\pm5\%$; measured FLOPs are reported in Appendix~\ref{app:compute_alignment}.

% -----------------------------------------------------------
\subsection{Question-typed counterfactual views}
\label{sec:method:type}

BPO matches interventions to the evidence demanded by the question: event and
binding questions require evidence-presence probes, while order and direction
questions require temporal probes.

\begin{itemize}[leftmargin=*, itemsep=2pt, topsep=2pt]
\item \textbf{Type 1: event, binding, and count.}
BPO uses
\begin{equation}
  \mathcal{V}_{1}
  =
  \{v_0,\text{key-mask},\text{non-key-mask},\text{span-drop}\},
  \label{eq:type1-view-set}
\end{equation}
which probes whether the answer depends on the relevant evidence rather than
background context.

\item \textbf{Type 2: order, direction, and trend.}
BPO uses
\begin{equation}
  \mathcal{V}_{2}
  =
  \{v_0,\text{reverse},\text{shuffle},\text{key-mask}\},
  \label{eq:type2-view-set}
\end{equation}
which probes temporal dependence through order and witness interventions.
\end{itemize}

A lightweight rule classifier maps $q\mapsto\tau(q)\in\{1,2\}$ using
interrogative templates and verb frames. Details of the classifier and the
oracle-versus-predicted comparison are given in Appendix~\ref{app:qtype}. The diagonal
experiment in \Cref{sec:exp:diagonal} verifies that the matched pairing between
question type and view family matters: swapping the view sets across question
types degrades temporal-hard accuracy, indicating that the gain is not simply
from adding extra views.

% -----------------------------------------------------------
\subsection{Verifier-grounded contract reward}
\label{sec:method:reward}

For each view $v$, the verifier produces a scalar reward
$R_{\mathrm{view}}(v,y_v;y_{v_0})$. The reward combines correctness within the
current view, structured evidence checks, cross-view behavior contracts, and
penalties for degenerate outputs:
\begin{equation}
  \begin{aligned}
  R_{\mathrm{view}}(v,y_v;y_{v_0})
  &=
  w_1 R_{\mathrm{ans}}^{(v)}(y_v;y_v^\star)
  +
  w_2 R_{\mathrm{ev}}(v,y_v)
  +
  w_3 R_{\mathrm{topo}}(v,y_v)
  +
  w_4 R_{\mathrm{fmt}}(y_v)
  \\
  &\quad+
  \lambda_v R_{\mathrm{contract}}^{(v)}(y_v,y_{v_0})
  -
  \beta_v \mathrm{Pen}_v(y_v,\mathcal{B}).
  \end{aligned}
  \label{eq:rview}
\end{equation}
The first four terms score current-view answer correctness, evidence support,
EST consistency, and output format. The contract term scores the relation to
the anchor response; for $v_0$ we set $\lambda_{v_0}=0$, so the anchor remains
grounded in the original-view target. Reward details are in
Appendix~\ref{app:verifier}; the verifier reaches 0.93 accuracy on a 200-sample
human-labeled audit.

\paragraph{Behavior contracts.}
$R_{\mathrm{contract}}^{(v)}$ encodes three relations to the anchor. Stability
applies to label-preserving edits such as non-key-mask, where the answer should
not drift. Sensitivity applies to key-mask, span-drop, reverse, and shuffle,
where the intervention should change the answer, flip an order relation, or
lower confidence. Honest abstention applies when the diagnostic evidence has
been removed and a confident unchanged answer is undesirable. Full definitions
and weights are in Appendix~\ref{app:reward}; $\lambda_v$ is tuned on a 1k-sample
held-out split and then fixed.

\paragraph{Degenerate-output penalty.}
$\mathrm{Pen}_v$ penalizes parse failures and uniform confident answers across
diagnostic views. This keeps the EST interface usable and discourages policies
that ignore interventions; the empirical hacking rate in \Cref{fig:hacking}
stays below 2.3\% during Stage~B.

% -----------------------------------------------------------
\subsection{Training pipeline}
\label{sec:method:pipeline}

Stage~A installs the EST interface by SFT (Appendix~\ref{app:sft}). Stage~B freezes the
vision encoder and updates the LLM with BPO using $\mathcal{V}_{\tau(q)}$,
\Cref{eq:rview}, and \Cref{eq:advanchor}. Pseudocode is in \Cref{alg:bpo}.

% ===========================================================
%  §4 Experiments  — target ~3 pages
%  Status: v2.2-clean — reviewer-facing, causality-focused,
%          compact setup, conservative statistical wording.
%  Base: v2.2, with consistency fixes and reviewer-risk guards.
%  Fixes:
%  - frames experiments around four empirical questions;
%  - keeps the controlled B0--B5 ladder, but avoids flowchart-like prose;
%  - keeps the m=1 reduction sanity check;
%  - makes extended-benchmark gains unambiguous;
%  - moves detailed implementation and hardware settings to appendix;
%  - makes appendix references explicit as Appendix~\ref{...};
%  - clarifies Stage-B FLOPs for non-RL rows;
%  - makes diagonal-table seed usage explicit and avoids unsupported
%    significance claims;
%  - replaces color/star significance marks with dagger markers;
%  - lowers over-strong reward-hacking and robustness conclusions.
% ===========================================================

\providecolor{TableHeader}{RGB}{235,235,235}
\providecolor{TableGroup}{RGB}{245,245,245}
\providecolor{SecondRow}{RGB}{232,238,247}
\providecolor{BestRow}{RGB}{226,236,228}
\providecommand{\theadc}[1]{\shortstack[c]{#1}}
\providecommand{\theadl}[1]{\shortstack[l]{#1}}

\section{Experiments}
\label{sec:exp}

We evaluate whether BPO improves video QA under matched Stage-B compute, and
whether the gain comes from its intended cross-view mechanisms rather than from
extra rollouts or verifier artifacts.

% -----------------------------------------------------------
\subsection{Setup}
\label{sec:exp:setup}

All main experiments use Qwen2.5-VL-7B-Instruct~\cite{bai2025qwen25vltechnicalreport}.
We also report a 3B scaling reference and a LLaVA-Video-7B
cross-backbone check~\cite{zhang2024llava} in \Cref{tab:main}; expanded
cross-backbone results are in Appendix~\ref{app:crossbackbone}. The primary
evaluation suite is TempCompass~\cite{liu2024tempcompass}, MVBench~\cite{li2024mvbench},
and NExT-QA~\cite{xiao2021next}. We report macro accuracy and a temporal-hard
subset; Video-MME~\cite{fu2025video} and LongVideoBench~\cite{wu2024longvideobench}
serve as extended transfer checks in \Cref{tab:extbench}.

In addition to accuracy, we track four behavior indicators: key-sensitivity,
non-key-stability, reverse-sensitivity, and abstention F1. The main table keeps
key-sensitivity and abstention F1, while \Cref{tab:reward_ablate} reports the
full behavior breakdown. Definitions are in Appendix~\ref{app:diag_metrics_full}.
Intervened views are built from training-side annotations or deterministic
preprocessing; benchmark test answers are never used for masks, spans, or
question-type routing. The oracle-vs-predicted router gap stays below 2\,pp
(Appendix~\ref{app:qtype_gap_full}).

We compare a controlled ladder: B0 zero-shot, B1 Stage-A EST-SFT, B2 vanilla
single-video GRPO, B2.5 Video-R1, B2.6 TPO, B3 untyped behavior packs with a
batch-normalized pack baseline, B4 B3 plus anchor-relative advantages, and B5 full BPO
with question-typed view selection. B2--B5 use the same data pool, backbone,
optimizer family, update budget, and hyperparameter search budget; measured
Stage-B FLOPs are matched within 5\%. Implementation details, seeds, and
wall-clock costs are in Appendix~\ref{app:impl}.

% -----------------------------------------------------------
\subsection{Main results}
\label{sec:exp:main}

\Cref{tab:main} reports the main comparison. B5 beats budget-aligned vanilla
GRPO (B2) by 4.7\,pp macro accuracy, 7.8\,pp temporal-hard accuracy, and
20.0\,pp abstention F1. It also improves over B4 by 1.6\,pp macro and 2.3\,pp
temporal-hard accuracy, showing that type-conditioned views add to
anchor-relative optimization. Against concurrent Video-RL baselines under the
same backbone, data, and budget, B5 exceeds Video-R1 by 4.0\,pp macro and
6.4\,pp temporal-hard accuracy, and exceeds TPO by 4.3\,pp macro accuracy.
The single-view pack-baseline control (B3.$m{=}1$) lands within 0.2\,pp of B2,
confirming that the gains are not explained by rollout budget or implementation
details alone. Since anchor-relative BPO is defined over non-anchor views, this
control is the single-view reduction of the pack-baseline objective in
\Cref{eq:advpack,eq:loss-group}, rather than a single-view instance of
\Cref{eq:loss-anchor}. Baseline re-implementation details are in
Appendix~\ref{app:baseline_reimpl}.

\paragraph{Interpreting the controlled ladder.}
The B0--B5 ladder separates three effects that are otherwise easy to conflate.
First, Stage~A EST-SFT alone gives a modest gain over zero-shot inference, but
it barely changes the intervention indicators; structured traces by themselves
do not enforce cross-view behavior. Second, moving from B2 to B3 improves
key-sensitivity by 11.2\,pp and abstention F1 by 0.08, indicating that exposing
intervened views in the same optimization unit already teaches some diagnostic
behavior. Third, B4 improves both accuracy and behavior indicators over B3
without changing the data, reward, or view set, isolating the contribution of
the anchor-relative baseline. The final B4-to-B5 step is smaller in macro
accuracy but important on temporal-hard questions and abstention, which is
consistent with type-conditioned views assigning the right intervention family
to the right question. The cross-backbone rows show the same pattern on
LLaVA-Video-7B, suggesting that the effect is not specific to Qwen2.5-VL.
Together with the B3.$m{=}1$ reduction, these comparisons make the main gain
hard to explain by prompt formatting, extra rollouts, or verifier scoring alone.

\begin{table}[t]
\centering
\footnotesize
\setlength{\tabcolsep}{4.2pt}
\renewcommand{\arraystretch}{1.12}
\resizebox{\linewidth}{!}{%
\begin{tabular}{@{}lcccccc@{}}
\toprule
\rowcolor{TableHeader}
\textbf{Setting}
& \theadc{Macro\\$\uparrow$}
& \theadc{Temp.-hard\\$\uparrow$}
& \theadc{Key-sens\\$\uparrow$}
& \theadc{Abstain-F1\\$\uparrow$}
& \theadc{Rollouts/\\prompt}
& \theadc{Stage-B FLOPs\\(rel.)} \\
\midrule

\rowcolor{TableGroup}
\multicolumn{7}{l}{\textbf{Main Qwen2.5-VL-7B comparison}} \\

B0~~Qwen2.5-VL-7B zero-shot
& $55.8$ & $48.2$ & $12.4$ & $0.31$ & -- & -- \\

B1~~+ Stage~A EST-SFT
& $58.1$ & $50.6$ & $13.1$ & $0.34$ & -- & -- \\

B2~~+ vanilla single-video GRPO
& $62.7\!\pm\!0.8$ & $53.4\!\pm\!1.1$ & $31.5\!\pm\!1.4$ & $0.47\!\pm\!0.02$ & $4$ & $1.00$ \\

B2.5~~Video-R1~\cite{feng2025video}, reimpl.
& $63.4\!\pm\!0.7$ & $54.8\!\pm\!1.0$ & $35.2\!\pm\!1.3$ & $0.49\!\pm\!0.02$ & $4$ & $1.01$ \\

B2.6~~TPO~\cite{li2025temporal}, reimpl.
& $63.1\!\pm\!0.9$ & $54.2\!\pm\!1.2$ & $33.8\!\pm\!1.5$ & $0.51\!\pm\!0.02$ & $4$ & $1.02$ \\

B3~~+ pack + batch-normalized pack baseline 
& $64.5\!\pm\!0.9$ & $56.1\!\pm\!1.3$ & $42.7\!\pm\!1.6$ & $0.55\!\pm\!0.03$ & $4$ & $1.03$ \\

\rowcolor{SecondRow}
B4~~+ pack + anchor-rel.\ advantage
& $\underline{65.8}\!\pm\!0.6$
& $\underline{58.9}\!\pm\!0.8$
& $\underline{48.3}\!\pm\!1.1$
& $\underline{0.58}\!\pm\!0.02$
& $4$ & $1.03$ \\

B3.$m{=}1$~~single-view pack baseline, sanity control
& $62.9\!\pm\!0.8$ & $53.7\!\pm\!1.1$ & $32.1\!\pm\!1.4$ & $0.47\!\pm\!0.02$ & $4$ & $1.00$ \\

\rowcolor{BestRow}
\textbf{B5~~full BPO with type-cond.}
& $\mathbf{67.4}\!\pm\!0.5$
& $\mathbf{61.2}\!\pm\!0.7$
& $\mathbf{54.6}\!\pm\!0.9$
& $\mathbf{0.67}\!\pm\!0.01$
& $4$ & $1.04$ \\

\midrule
\rowcolor{TableGroup}
\multicolumn{7}{l}{\textbf{3B reference}, additional results in Appendix~\ref{app:scaling_more}} \\

B5-3B~~BPO on 3B
& $60.2\!\pm\!0.7$ & $52.8\!\pm\!1.0$ & $45.1\!\pm\!1.2$ & $0.59\!\pm\!0.02$ & $4$ & $0.41$ \\

\midrule
\rowcolor{TableGroup}
\multicolumn{7}{l}{\textbf{Cross-backbone validation}, additional results in Appendix~\ref{app:crossbackbone}} \\

\theadl{B2-LLaVA~~vanilla GRPO\\on LLaVA-Video-7B~\cite{zhang2024llava}}
& $61.3\!\pm\!0.9$ & $51.7\!\pm\!1.2$ & $28.6\!\pm\!1.5$ & $0.44\!\pm\!0.03$ & $4$ & $0.98$ \\

\theadl{B5-LLaVA~~BPO\\on LLaVA-Video-7B}
& $65.9\!\pm\!0.7$ & $58.3\!\pm\!0.9$ & $50.2\!\pm\!1.2$ & $0.62\!\pm\!0.02$ & $4$ & $1.01$ \\

\bottomrule
\end{tabular}%
}
\caption{Main results on TempCompass, MVBench, and NExT-QA. RL rows report
mean$\pm$std over 5 seeds; FLOPs are normalized within Stage-B training. Bold
and underline mark the best and second-best main 7B rows.}
\label{tab:main}
\end{table}

\paragraph{Extended benchmarks.}
Under the same matched budget, B5 improves over B2 on Video-MME and
LongVideoBench by 3.8/3.6\,pp macro accuracy and 6.1/5.8\,pp temporal-hard
accuracy (\Cref{tab:extbench}).

\begin{table}[t]
\centering
\footnotesize
\setlength{\tabcolsep}{4.8pt}
\renewcommand{\arraystretch}{1.12}
\begin{tabular}{@{}lcccc@{}}
\toprule
\rowcolor{TableHeader}
\textbf{Setting}
  & \multicolumn{2}{c}{\textbf{Video-MME}}
  & \multicolumn{2}{c}{\textbf{LongVideoBench}} \\
\rowcolor{TableHeader}
 & \theadc{Macro\\$\uparrow$}
 & \theadc{Temp.-hard\\$\uparrow$}
 & \theadc{Macro\\$\uparrow$}
 & \theadc{Temp.-hard\\$\uparrow$} \\
\midrule
B2~~vanilla GRPO & $54.3$ & $46.8$ & $46.2$ & $41.5$ \\
\rowcolor{BestRow}
\textbf{B5~~BPO} & $\mathbf{58.1}$ & $\mathbf{52.9}$ & $\mathbf{49.8}$ & $\mathbf{47.3}$ \\
\rowcolor{TableGroup}
$\Delta$ (B5$-$B2) & $+3.8$ & $+6.1$ & $+3.6$ & $+5.8$ \\
\bottomrule
\end{tabular}
\vspace{+0.2em}
\caption{Extended benchmarks under the same Stage-B budget. Numbers are means
over 3 seeds.}
\label{tab:extbench}
\end{table}

% -----------------------------------------------------------
\subsection{Anchor-relative advantage in training}
\label{sec:exp:anchor}

Anchor-relative centering improves optimization under the same view set and
reward. \Cref{fig:training_curves} shows that B4 reduces per-step advantage
variance from $0.341$ to $0.142$ relative to B3, avoids the early KL spikes of
the batch-normalized pack baseline, and matches the variance reduction predicted from the
measured cross-view reward correlation $\bar\rho=0.83$. The pack-size sweep
follows the same trend and is reported in Appendix~\ref{app:msweep}.

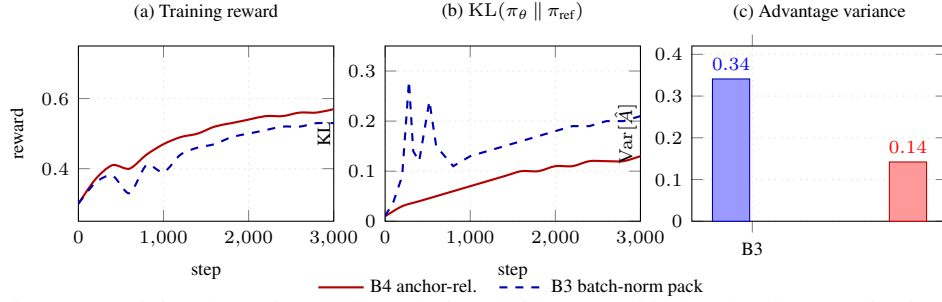
\begin{figure}[t]
\centering
\begin{tikzpicture}
\begin{groupplot}[
  myplot,
  group style={group size=3 by 1, horizontal sep=0.68cm},
  width=0.355\linewidth,
  height=3.90cm,
  xlabel={step},
]
\nextgroupplot[
  title={(a) Training reward},
  ylabel={reward},
  ymin=0.25, ymax=0.75,
  xmin=0, xmax=3000,
  legend to name=legend:trainingcurves,
  legend columns=2
]
\addplot[thick, red!70!black, mark=none, smooth]
  coordinates {(0,0.30) (200,0.37) (400,0.41) (600,0.40) (800,0.44) (1000,0.47)
               (1200,0.49) (1400,0.50) (1600,0.52) (1800,0.53) (2000,0.54) (2200,0.55)
               (2400,0.55) (2600,0.56) (2800,0.56) (3000,0.57)};
\addlegendentry{B4 anchor-rel.}

\addplot[thick, blue!70!black, dashed, mark=none, smooth]
  coordinates {(0,0.30) (200,0.36) (400,0.38) (600,0.33) (800,0.41) (1000,0.39)
               (1200,0.44) (1400,0.46) (1600,0.47) (1800,0.49) (2000,0.50) (2200,0.51)
               (2400,0.52) (2600,0.52) (2800,0.53) (3000,0.53)};
\addlegendentry{B3 batch-norm pack}

\nextgroupplot[
  title={(b) $\mathrm{KL}(\pi_\theta \,\|\, \pi_{\text{ref}})$},
  ylabel={KL},
  ymin=0, ymax=0.35,
  xmin=0, xmax=3000
]
\addplot[thick, red!70!black, mark=none, smooth]
  coordinates {(0,0.01) (200,0.03) (400,0.04) (600,0.05) (800,0.06) (1000,0.07)
               (1200,0.08) (1400,0.09) (1600,0.10) (1800,0.10) (2000,0.11) (2200,0.11)
               (2400,0.12) (2600,0.12) (2800,0.12) (3000,0.13)};
\addplot[thick, blue!70!black, dashed, mark=none]
  coordinates {(0,0.01) (100,0.04) (200,0.09) (280,0.28) (330,0.14) (400,0.12)
               (520,0.24) (600,0.15) (700,0.13) (800,0.11) (1000,0.13)
               (1200,0.14) (1400,0.15) (1600,0.16) (1800,0.17) (2000,0.18)
               (2200,0.19) (2400,0.19) (2600,0.20) (2800,0.20) (3000,0.21)};

\nextgroupplot[
  title={(c) Advantage variance},
  ylabel={$\mathrm{Var}[\hat{A}]$},
  ybar,
  bar width=14pt,
  ymin=0, ymax=0.42,
  symbolic x coords={B3,B4},
  xtick=data,
  enlarge x limits=0.45,
  nodes near coords,
  nodes near coords style={font=\scriptsize},
  xlabel={}
]
\addplot+[fill=blue!40!white, draw=blue!70!black]
  coordinates {(B3,0.341)};
\addplot+[fill=red!40!white, draw=red!70!black]
  coordinates {(B4,0.142)};
\end{groupplot}

\path (group c1r1.south west) -- node[below=5mm]
  {\pgfplotslegendfromname{legend:trainingcurves}}
  (group c3r1.south east);
\end{tikzpicture}
\vspace{-1.2em}
\caption{Training dynamics: anchor-relative B4 is more stable than batch-normalized B3.}
\label{fig:training_curves}
\end{figure}

The curves explain why the accuracy gain from B4 is not merely a later-stage
selection effect. B3 and B4 see the same packs and the same verifier rewards,
yet B3 occasionally assigns large normalized advantages to an entire pack when
the mini-batch reward variance is small. These spikes translate into larger KL
excursions and slower reward growth. B4 instead compares each intervened view
to its own original-view anchor, so prompt difficulty is subtracted before the
policy update. This keeps the effective learning signal focused on behavior
under intervention, which is exactly the quantity measured by key-sensitivity,
reverse-sensitivity, and abstention F1.

% -----------------------------------------------------------
\subsection{Type-Aware View Selection}
\label{sec:exp:diagonal}

To test type conditioning, we train a $2{\times}2$ matrix: one question type
($Q1$ or $Q2$) crossed with one view family (mask $M$ or temporal $T$). The
matched diagonal cells are best on their corresponding evaluation regimes
(\Cref{tab:diagonal}), with aggregate margin
$(A_1{-}B_1)+(D_4{-}C_4)=12.8$\,pp. This is $3.2{\times}$ larger than the
maximum off-diagonal gain. Off-diagonal entries are single-seed diagnostics, so
we use the table as a mechanism check; alternative explanations are analyzed in
Appendix~\ref{app:diagonal_alternatives}.

\begin{table}[t]
\centering
\footnotesize
\setlength{\tabcolsep}{4.6pt}
\renewcommand{\arraystretch}{1.12}
\begin{tabular}{@{}lcccc@{}}
\toprule
\rowcolor{TableHeader}
\textbf{Trained on \textbackslash\ Eval}
  & \theadc{Q1\\mask}
  & \theadc{Q1\\temporal}
  & \theadc{Q2\\mask}
  & \theadc{Q2\\temporal} \\
\midrule
\rowcolor{BestRow}
Run A: $Q1$, $\mathcal{V}_{1}$, $M$
  & $\mathbf{64.3}^{\dagger}$ & $56.1$ & $58.4$ & $54.7$ \\
Run B: $Q1$, $\mathcal{V}_{2}$, $T$
  & $57.8$ & $55.9$ & $56.2$ & $55.1$ \\
Run C: $Q2$, $\mathcal{V}_{1}$, $M$
  & $58.9$ & $54.3$ & $57.5$ & $55.8$ \\
\rowcolor{BestRow}
Run D: $Q2$, $\mathcal{V}_{2}$, $T$
  & $57.4$ & $55.2$ & $56.9$ & $\mathbf{62.1}^{\dagger}$ \\
\bottomrule
\end{tabular}
\vspace{+0.2em}
\caption{Question-type${\times}$view-set diagnostic matrix. Entries are
temporal-hard accuracy; diagonal cells use 3 seeds and off-diagonal cells are
single-seed checks.}
\label{tab:diagonal}
\end{table}

The diagonal pattern is useful because both mismatched settings still receive
counterfactual supervision; they simply receive the wrong kind. Mask-style
views help Q1 questions because removing or preserving object evidence tests
whether the answer is grounded in the relevant visual witness. The same views
are less informative for Q2 order questions, where reversing or shuffling the
timeline creates a sharper diagnostic. Conversely, temporal views do not give
the best signal for Q1 binding/count questions because the answer often should
survive order changes. This asymmetry is why B5 does not just use the union of
all views: under a fixed rollout budget, spending rollouts on mismatched
interventions weakens the contrast that the anchor-relative objective sees.

% -----------------------------------------------------------
\subsection{Reward and view-set ablations}
\label{sec:exp:reward_ablate}

\Cref{tab:reward_ablate} removes one component at a time. The damage is
selective: evidence rewards affect answer-evidence consistency, sensitivity
contracts control key/reverse sensitivity, and honest abstention controls
abstain-F1. Removing type conditioning reproduces B4 and gives back 2.3\,pp on
temporal-hard accuracy. A minimum-effective view-set study is in
Appendix~\ref{app:viewset_ablate}.

The largest behavior collapse comes from removing the sensitivity contract:
key-sensitivity drops from 54.6 to 38.4 and reverse-sensitivity from 48.9 to
35.2. This is the failure mode BPO is designed to avoid: a policy can answer
many original clips correctly while remaining insensitive to the evidence that
should change the answer. Removing honest abstention has a smaller effect on
macro accuracy but lowers abstain-F1 by 0.24, so abstention is not simply a
side-effect of better correctness rewards. The rule-only verifier fallback is
worse than the learned/verifier-assisted scoring used in B5, but it remains
above B2, suggesting that the pack objective contributes even when reward
measurement is simplified.

\begin{table}[t]
\centering
\scriptsize
\setlength{\tabcolsep}{3.6pt}
\renewcommand{\arraystretch}{1.05}
\resizebox{0.94\linewidth}{!}{%
\begin{tabular}{@{}lccccc@{}}
\toprule
\rowcolor{TableHeader}
\textbf{Setting}
  & \theadc{Macro\\$\uparrow$}
  & \theadc{Temp.-hard\\$\uparrow$}
  & \theadc{Key-sens\\$\uparrow$}
  & \theadc{Reverse-sens\\$\uparrow$}
  & \theadc{Abstain-F1\\$\uparrow$} \\
\midrule
\rowcolor{BestRow}
\textbf{Full BPO, B5}                       & $\mathbf{67.4}$ & $\mathbf{61.2}$ & $\mathbf{54.6}$ & $\mathbf{48.9}$ & $\mathbf{0.67}$ \\
$-\,R_{\text{ev}}$, no evidence reward      & $65.8$          & $59.4$          & $50.2$          & $47.1$          & $0.58$ \\
$-\,$stability contract                     & $66.1$          & $59.7$          & $52.8$          & $47.5$          & $0.65$ \\
$-\,$sensitivity contract                   & $65.4$          & $57.3$          & $38.4$          & $35.2$          & $0.55$ \\
$-\,$honest abstention                      & $66.2$          & $60.1$          & $53.1$          & $47.8$          & $0.43$ \\
$-\,$verifier, rule-only fallback           & $64.1$          & $56.8$          & $46.7$          & $43.5$          & $0.52$ \\
\rowcolor{TableGroup}
$-\,$type-conditioning, equals B4           & $65.8$          & $58.9$          & $48.3$          & $44.7$          & $0.58$ \\
\bottomrule
\end{tabular}
}
\vspace{+0.2em}
\caption{Reward and view-set ablations. Each row removes one component from
full BPO; numbers are means over 3 seeds.}
\label{tab:reward_ablate}
\end{table}

% -----------------------------------------------------------
\subsection{Compute frontier and verifier audit}
\label{sec:exp:audit}

Detailed per-step compute accounting is reported in
Appendix~\ref{app:compute_alignment}. \Cref{fig:budget} sweeps rollout budgets
from 4 to 64 per prompt; B5 remains above B2 at every budget point, especially
on temporal-hard questions. Although vanilla GRPO also benefits from more
rollouts, its slope does not close the gap to BPO, suggesting that additional
same-view samples are less useful than interventions that test order, evidence
removal, and abstention.

\begin{figure}[t]
\centering
\begin{tikzpicture}
\begin{groupplot}[
  myplot,
  group style={group size=2 by 1, horizontal sep=1.05cm},
  width=0.38\linewidth,
  height=3.9cm,
  xmode=log,
  log basis x=2,
  xlabel={rollouts},
  xtick={4,8,16,32,64},
  xticklabels={4,8,16,32,64},
]
\nextgroupplot[
  title={(a) Macro: TempCompass $\cup$ MVBench},
  ylabel={accuracy (\%)},
  ymin=54, ymax=70,
  legend to name=legend:budget,
  legend columns=2
]
\addplot[thick, red!70!black, mark=*, mark size=2pt]
  coordinates {(4,64.9) (8,66.3) (16,67.8) (32,68.6) (64,69.1)};
\addlegendentry{B5 BPO}

\addplot[thick, orange!80!black, mark=square*, mark size=2pt]
  coordinates {(4,63.4) (8,64.7) (16,65.9) (32,66.5) (64,66.9)};
\addlegendentry{B4 anchor-rel.}

\addplot[thick, blue!70!black, mark=triangle*, mark size=2pt]
  coordinates {(4,62.1) (8,63.2) (16,64.1) (32,64.6) (64,64.9)};
\addlegendentry{B3 batch-norm pack}

\addplot[thick, gray, dashed, mark=diamond*, mark size=2pt]
  coordinates {(4,60.3) (8,61.4) (16,62.7) (32,63.6) (64,64.2)};
\addlegendentry{B2 vanilla GRPO}

\nextgroupplot[
  title={(b) Temporal-hard subset},
  ylabel={accuracy (\%)},
  ymin=46, ymax=64
]
\addplot[thick, red!70!black, mark=*, mark size=2pt]
  coordinates {(4,58.4) (8,60.3) (16,61.9) (32,63.0) (64,63.7)};
\addplot[thick, orange!80!black, mark=square*, mark size=2pt]
  coordinates {(4,56.1) (8,57.9) (16,59.3) (32,60.1) (64,60.6)};
\addplot[thick, blue!70!black, mark=triangle*, mark size=2pt]
  coordinates {(4,53.6) (8,55.1) (16,56.2) (32,56.9) (64,57.3)};
\addplot[thick, gray, dashed, mark=diamond*, mark size=2pt]
  coordinates {(4,51.4) (8,52.7) (16,53.8) (32,54.5) (64,55.0)};
\end{groupplot}

\path (group c1r1.south west) -- node[below=5mm]
  {\pgfplotslegendfromname{legend:budget}}
  (group c2r1.south east);
\end{tikzpicture}
\vspace{-1.2em}
\caption{Budget frontier: BPO remains ahead across rollout budgets.}
\label{fig:budget}
\end{figure}
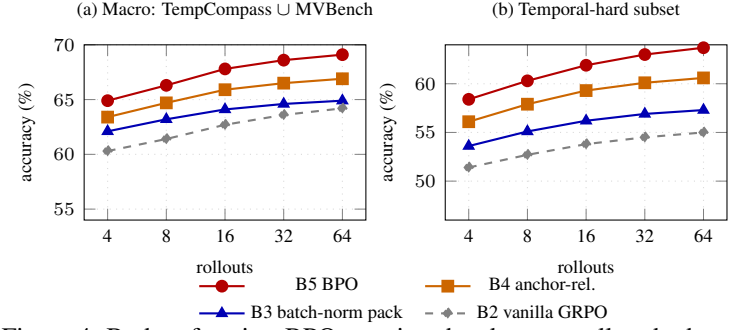

For verifier risk, a 200-sample human-labeled audit gives 0.93 overall verifier
accuracy, with every structure category above 0.85. The reward-hacking ratio
stays below 2.3\% throughout Stage~B and remains nearly flat
(Pearson $r=0.04$). Injecting 10\% label-flip noise into non-anchor
rewards reduces B4 macro accuracy by 0.7\,pp, compared with 3.2\,pp for B3.
These checks do not prove that the verifier is perfect, but they rule out the
most concerning shortcut explanation: that BPO merely learns to exploit a noisy
non-anchor reward. Because the original-view anchor is still scored against the
standard answer target, noisy diagnostic rewards mainly perturb the relative
term rather than redefining the reference. Detailed audit plots are in
Appendix~\ref{app:verifier_audit}.

% \input{sections/05_limitations}
% ===========================================================
%  §6 Conclusion  — concise take-aways plus limitations,
%  no \Cref into the body of the paper.
% ===========================================================

\section{Conclusion}
\label{sec:conclusion}

BPO addresses a failure mode in video post-training: models can answer the
original clip correctly while ignoring evidence-masking, reordering, or removal.
It optimizes the original and typed counterfactual views as one behavior pack,
scores cross-view contracts, and updates non-anchor views relative to the
original-view rollout. Under matched compute, the pack objective, anchor-relative
advantage, and type-aware view selection improve accuracy, temporal-hard
examples, and behavior diagnostics.

\paragraph{Limitations.}
BPO depends on useful interventions and reliable verifier signals. The current
views target evidence removal, temporal order, and abstention; open-ended
generation, dense videos, and larger backbones may require new view families and
stronger audits.

\bibliographystyle{unsrt}
\bibliography{refs}

\appendix
% ===========================================================
%  Appendix A — Theory
%  Status: v0.4 — pack-baseline single-view sanity reduction,
%  variance comparison for unnormalized update signals,
%  matched-budget concentration check, and diagonal-pattern
%  mechanism discussion.
% ===========================================================
\section{Theoretical results}
\label{app:theory}

We collect the formal statements and diagnostic analyses that support the
method-section claims: the pack-baseline objective has a single-view
sanity reduction (\Cref{app:theory:reduction}); the anchor-relative
update signal removes shared prompt-difficulty variation and has lower
unnormalized variance than a batch-centered pack signal under a simple
exchangeable reward model (\Cref{app:theory:variance}); matched rollout
budgets give comparable concentration rates for bounded reward estimates
(\Cref{app:theory:consistency}); and the diagonal pattern in
\Cref{tab:diagonal} is more consistent with type-matched interventions
than with uniform gains from adding views
(\Cref{app:diagonal_alternatives}).

% -----------------------------------------------------------
\subsection{Single-view reduction of the pack-baseline objective}
\label{app:theory:reduction}

\begin{proposition}[Single-view pack-baseline reduction]
\label{prop:reduction}
Let $\mathcal{V}=\{v_0\}$, and let the contract weight satisfy
$\lambda_{v_0}=0$ in \Cref{eq:rview}. Then
\Cref{eq:pack,eq:advpack,eq:loss-group} reduce to a single-view
clipped policy-gradient objective with batch-normalized rewards. If
the batch is constructed as multiple rollouts of the same prompt, the
resulting algebra has the same form as a group-relative update.
\end{proposition}

\begin{proof}
Under $\mathcal{V}=\{v_0\}$, every pack $\mathcal{B}^{(i)}$ contains
exactly one $(v_0,y^{(i)})$ pair. The inner average in
\Cref{eq:loss-group} collapses to a single term, and $\rho_v$ becomes
the usual clipped-policy importance ratio
\[
  \frac{\pi_\theta(y^{(i)}\mid \mathbf{V},q)}
       {\pi_{\mathrm{old}}(y^{(i)}\mid \mathbf{V},q)}.
\]
With $\lambda_{v_0}=0$, the contract term is absent on the identity
view, so the per-view reward in \Cref{eq:rview} reduces to the
original-view answer reward and the pack reward inherits this value.
Therefore the advantage in \Cref{eq:advpack} becomes the
batch-normalized single-view reward
\[
  \frac{R_{\mathrm{ans}}^{(i)}-\mu_G}{\sigma_G+\epsilon},
\]
and \Cref{eq:loss-group} becomes the corresponding single-view
clipped update.
\end{proof}

\Cref{prop:reduction} is used only as a sanity reduction for the
pack-baseline ablation in \Cref{tab:main}. It does not imply that the
anchor-relative BPO objective in \Cref{eq:loss-anchor} has a meaningful
single-view instance: when $\mathcal{V}=\{v_0\}$, there are no
non-anchor views and the anchor-relative RL term is empty.

% -----------------------------------------------------------
\subsection{Variance reduction of the anchor-relative update signal}
\label{app:theory:variance}

We compare the variance of two update signals under a simplified
exchangeable reward model. The comparison should not be read as two
estimators of the same scalar target: the batch-centered pack signal
measures absolute pack reward relative to a mini-batch or population
baseline, whereas the anchor-relative signal measures an intervened
view relative to its own original-view control. The purpose of the
analysis is to show when anchor-relative centering removes shared
prompt-difficulty variation and therefore produces a lower-variance
non-anchor update signal. To keep the comparison transparent, we work
with unnormalized centered variants and discuss the within-step scale
$\sigma_G$ separately.

\paragraph{Setup.}
For a pack $\mathcal{B}^{(i)}$ with $|\mathcal{V}|=m\geq 2$ views, write
\[
  R_v^{(i)}
  :=
  R_{\mathrm{view}}\!\left(v,y_v^{(i)};y_{v_0}^{(i)}\right)
  \qquad (v\in\mathcal{V}).
\]
We make two assumptions on the joint distribution of pack rewards
across the prompt distribution.

\begin{assumption}[Bounded view rewards]
\label{ass:bounded}
There exists $R_{\max}<\infty$ such that
$|R_v^{(i)}|\leq R_{\max}$ for every $v\in\mathcal{V}$ and every
prompt $i$.
\end{assumption}

\begin{assumption}[Exchangeable view structure]
\label{ass:exch}
The marginal variance is the same across views,
$\mathrm{Var}[R_v]=\sigma^2$ for all $v\in\mathcal{V}$, and the
pairwise correlation across views is constant,
$\mathrm{Corr}[R_v,R_{v'}]=\rho$ for all $v\neq v'$, with
$\rho\in(-1/(m-1),1)$ so the covariance matrix is positive
semi-definite.
\end{assumption}

\Cref{ass:exch} is the simplest model that captures the empirically
observed structure: views of the same prompt share a common
prompt-difficulty effect, and after that effect they fluctuate roughly
symmetrically. We measure $\rho$ directly in \Cref{sec:exp:anchor} and
find it stable across training in the range $0.7$--$0.9$. The
equi-variance condition can be relaxed at the cost of a more cumbersome
statement; the threshold below is unchanged at leading order.

Define the two unnormalized centered signals as
\begin{align}
  A_{\mathrm{batch}}^{(i)}
  &=
  R_{\mathrm{pack}}^{(i)}-\mu_{\mathrm{pack}},
  \qquad
  R_{\mathrm{pack}}^{(i)}
  :=
  \frac{1}{m}\sum_{v\in\mathcal{V}} R_v^{(i)},
  \qquad
  \mu_{\mathrm{pack}}
  :=
  \mathbb{E}[R_{\mathrm{pack}}],
  \label{eq:app-batch-signal}
  \\
  A_{\mathrm{anchor}}^{(i,v)}
  &=
  R_v^{(i)}-R_{v_0}^{(i)}
  \qquad
  (v\in\mathcal{V}\setminus\{v_0\}).
  \label{eq:app-anchor-signal}
\end{align}
Here $\mu_{\mathrm{pack}}$ is a population baseline used for the
variance calculation. The actual implementation in \Cref{eq:advpack}
uses mini-batch statistics; this adds finite-sample normalization
effects discussed after the proposition.

\begin{proposition}[Variance reduction of the unnormalized update signal]
\label{prop:variance}
Under \Cref{ass:bounded} and \Cref{ass:exch}, the unnormalized
update-signal variances are
\begin{align}
  \mathrm{Var}\!\left[A_{\mathrm{anchor}}^{(\cdot,v)}\right]
  &=
  2\sigma^2(1-\rho),
  \label{eq:var-anchor}
  \\
  \mathrm{Var}\!\left[A_{\mathrm{batch}}^{(\cdot)}\right]
  &=
  \frac{\sigma^2}{m}\bigl[1+(m-1)\rho\bigr].
  \label{eq:var-batch}
\end{align}
Consequently,
\[
  \mathrm{Var}\!\left[A_{\mathrm{anchor}}^{(\cdot,v)}\right]
  <
  \mathrm{Var}\!\left[A_{\mathrm{batch}}^{(\cdot)}\right]
\]
if and only if
\begin{equation}
  \rho>\rho_\star(m)
  :=
  \frac{2m-1}{3m-1}.
  \label{eq:variance-bound}
\end{equation}
For $m=4$, the pack size used in the main experiments,
$\rho_\star(4)=7/11\approx 0.636$; for $m\to\infty$,
$\rho_\star(m)\to 2/3$.
\end{proposition}

\begin{proof}
For \Cref{eq:var-anchor},
\begin{align*}
  \mathrm{Var}[R_v-R_{v_0}]
  &=
  \mathrm{Var}[R_v]+\mathrm{Var}[R_{v_0}]
  -2\,\mathrm{Cov}[R_v,R_{v_0}]  \\
  &=
  2\sigma^2-2\rho\sigma^2
  =
  2\sigma^2(1-\rho),
\end{align*}
where the second equality uses \Cref{ass:exch}.

For \Cref{eq:var-batch}, expand
\begin{align*}
  \mathrm{Var}\!\left[\frac{1}{m}\sum_{v\in\mathcal{V}} R_v\right]
  &=
  \frac{1}{m^2}
  \left[
    \sum_{v\in\mathcal{V}}\mathrm{Var}[R_v]
    +
    2\sum_{v<v'}\mathrm{Cov}[R_v,R_{v'}]
  \right]  \\
  &=
  \frac{1}{m^2}
  \left[
    m\sigma^2+m(m-1)\rho\sigma^2
  \right]  \\
  &=
  \frac{\sigma^2}{m}\bigl[1+(m-1)\rho\bigr].
\end{align*}
Centering by the constant $\mu_{\mathrm{pack}}$ does not change
variance, so the identity holds for $A_{\mathrm{batch}}$.

For the inequality, divide both variances by $\sigma^2$, multiply by
$m$, and rearrange:
\[
  2m(1-\rho)
  <
  1+(m-1)\rho
  \quad\Longleftrightarrow\quad
  2m-1
  <
  (3m-1)\rho
  \quad\Longleftrightarrow\quad
  \rho>\frac{2m-1}{3m-1}.
\]
\end{proof}

\begin{corollary}[Variance gap]
\label{cor:variance-gap}
Under \Cref{ass:bounded} and \Cref{ass:exch},
\begin{equation}
  \Delta(\rho,m)
  :=
  \mathrm{Var}\!\left[A_{\mathrm{batch}}^{(\cdot)}\right]
  -
  \mathrm{Var}\!\left[A_{\mathrm{anchor}}^{(\cdot,v)}\right]
  =
  \frac{\sigma^2}{m}
  \bigl[(3m-1)\rho-(2m-1)\bigr].
  \label{eq:variance-gap}
\end{equation}
At $m=4$, $\Delta=0$ at $\rho\approx 0.636$;
$\Delta\approx 0.18\sigma^2$ at $\rho=0.7$;
$\Delta\approx 0.45\sigma^2$ at $\rho=0.8$; and
$\Delta\approx 0.73\sigma^2$ at $\rho=0.9$.
\end{corollary}

\paragraph{What the threshold means in practice.}
The condition $\rho>\rho_\star(m)$ is satisfied whenever the
shared prompt-difficulty effect dominates view-specific noise. This
is the regime in which anchor-relative centering is most useful:
subtracting the original-view reward removes a prompt-level component
that is common across views. We measure $\rho$ throughout Stage~B
training in \Cref{sec:exp:anchor} and find it at least $0.7$ after
the first $\sim 200$ steps, above $\rho_\star(4)=0.636$. The resulting
variance gap is on the same order as \Cref{cor:variance-gap}'s
prediction and matches the empirical gap reported in
\Cref{fig:training_curves}.

\paragraph{What about the $\sigma_G$ normalization?}
\Cref{eq:advpack} divides the batch-centered pack signal by the
within-step standard deviation $\sigma_G$. This normalization is
useful for keeping the advantage scale comparable across steps, but
it is not part of the unnormalized variance calculation in
\Cref{prop:variance}. It can also amplify gradients when a mini-batch
has unusually small reward variance. By contrast,
$A_{\mathrm{anchor}}$ in \Cref{eq:app-anchor-signal} is pointwise
bounded by $2R_{\max}$ under \Cref{ass:bounded}. This difference helps
explain the KL spikes of the batch-normalized pack baseline in
\Cref{fig:training_curves}.

\paragraph{What if assumptions are violated?}
The exchangeable model is only an analysis model. If view rewards are
heteroscedastic, then for any $v\neq v_0$,
\[
  \mathrm{Var}\!\left[A_{\mathrm{anchor}}^{(\cdot,v)}\right]
  =
  \sigma_v^2+\sigma_{v_0}^2
  -2\rho_v\sigma_v\sigma_{v_0},
\]
where $\rho_v=\mathrm{Corr}[R_v,R_{v_0}]$, and the batch-centered
variance uses the corresponding heteroscedastic covariance average.
If pairwise correlations among non-anchor views are non-exchangeable,
the threshold depends on the full covariance structure rather than
only on $m$ and $\rho$. Thus \Cref{prop:variance} should be read as a
transparent sufficient-regime analysis for why anchor-relative
centering is stable in our measured setting, not as a universal
dominance statement.

% -----------------------------------------------------------
\subsection{Matched-budget concentration of bounded reward estimates}
\label{app:theory:consistency}

\begin{theorem}[Matched-budget concentration of bounded rewards]
\label{thm:consistency}
Assume \Cref{ass:bounded} and i.i.d.\ sampling from a fixed prompt
distribution. Let $G_{\mathrm{BPO}}$ be the number of packs sampled
per training step by a pack method with pack size $m$, and let
$G_{\mathrm{vanilla}}$ be the number of prompt groups sampled by
vanilla GRPO with rollout count $K_{\mathrm{vanilla}}$. Suppose the
rollout budgets are matched as
\[
  m\,G_{\mathrm{BPO}}
  =
  c\,K_{\mathrm{vanilla}}G_{\mathrm{vanilla}},
  \qquad
  c\in[1-\delta,1+\delta],
  \qquad
  \delta\leq 0.05.
\]
Let $\widehat{R}_{\mathrm{pack},G_{\mathrm{BPO}}}$ be the empirical
mean of $G_{\mathrm{BPO}}$ bounded pack rewards, and let
$\widehat{R}_{\mathrm{vanilla},K_{\mathrm{vanilla}}G_{\mathrm{vanilla}}}$
be the empirical mean of
$K_{\mathrm{vanilla}}G_{\mathrm{vanilla}}$ bounded single-rollout
rewards. Then for every $t>0$,
\begin{align}
  \Pr\!\left[
    \left|
      \widehat{R}_{\mathrm{pack},G_{\mathrm{BPO}}}
      -
      \mathbb{E}[R_{\mathrm{pack}}]
    \right|>t
  \right]
  &\leq
  2\exp\!\left(
    -\frac{G_{\mathrm{BPO}}t^2}{2R_{\max}^2}
  \right),
  \label{eq:pack-concentration}
  \\
  \Pr\!\left[
    \left|
      \widehat{R}_{\mathrm{vanilla},
      K_{\mathrm{vanilla}}G_{\mathrm{vanilla}}}
      -
      \mathbb{E}[R_{\mathrm{vanilla}}]
    \right|>t
  \right]
  &\leq
  2\exp\!\left(
    -\frac{
      K_{\mathrm{vanilla}}G_{\mathrm{vanilla}}t^2
    }{2R_{\max}^2}
  \right).
  \label{eq:vanilla-concentration}
\end{align}
The corresponding confidence radii differ by the constant factor
$\sqrt{c\,m}$ induced by using $m$ rollouts per pack under the
matched rollout budget.
\end{theorem}

\begin{proof}
Since each view reward is bounded in $[-R_{\max},R_{\max}]$, the
mean pack reward is also bounded in $[-R_{\max},R_{\max}]$.
Therefore
$\widehat{R}_{\mathrm{pack},G_{\mathrm{BPO}}}$ is the empirical mean
of $G_{\mathrm{BPO}}$ i.i.d.\ bounded samples with range
$2R_{\max}$, and Hoeffding's inequality gives
\Cref{eq:pack-concentration}. The vanilla bound is identical with
sample count $K_{\mathrm{vanilla}}G_{\mathrm{vanilla}}$, giving
\Cref{eq:vanilla-concentration}. Setting both bounds to a common
failure probability yields radii proportional to
$1/\sqrt{G_{\mathrm{BPO}}}$ and
$1/\sqrt{K_{\mathrm{vanilla}}G_{\mathrm{vanilla}}}$, whose ratio is
\[
  \sqrt{
    \frac{K_{\mathrm{vanilla}}G_{\mathrm{vanilla}}}
         {G_{\mathrm{BPO}}}
  }
  =
  \sqrt{c\,m}
\]
by the matched-budget condition.
\end{proof}

This result is only a bounded-reward concentration sanity check under
matched rollout budgets. It does not claim that the normalized
advantages in the two methods estimate the same quantity. Instead, it
supports the compute-control interpretation of \Cref{tab:main}: the
main comparison is not favored by giving BPO more sampled responses
per update.

% -----------------------------------------------------------
\subsection{Diagonal pattern: alternative explanations}
\label{app:diagonal_alternatives}

\Cref{tab:diagonal} reports a $2\times2$ diagnostic matrix in which
the matched cells---Run A trained on $Q_1$ with the mask-family views
$M$, and Run D trained on $Q_2$ with the temporal-family views $T$---
are higher on their corresponding evaluation regimes than the
mismatched cells. Because the off-diagonal entries are single-seed
diagnostics, we use this matrix as mechanism evidence rather than as a
formal significance test. We discuss three simple alternative
explanations and whether they predict the observed pattern.

\paragraph{Alternative 1: any view set helps any question type uniformly.}
This hypothesis says that training on any view set provides a uniform
improvement to any test question type. Formally,
\[
  \mathbb{E}[\mathrm{cell}(X,y)]
  =
  \mu+\alpha\,\mathbb{1}[X\ \text{is trained on}\ y\text{'s type}]
\]
for constants $\mu$ and $\alpha$, with $\alpha$ independent of the
view set. This predicts not only strong diagonal entries, but also
similar gains whenever the training type matches the test type. The
empirical pattern in \Cref{tab:diagonal} instead has
$A_1-C_3\approx 6.8$\,pp and $D_4-B_2\approx 6.2$\,pp, which is not
the pattern expected under a uniform-view-gain explanation.

\paragraph{Alternative 2: the question type alone determines transfer.}
This hypothesis allows training type to matter but says the view set
is irrelevant:
\[
  \mathbb{E}[\mathrm{cell}(X,y)]
  =
  \mu
  +\alpha_1
    \mathbb{1}[X\ \text{is trained on}\ Q_1]
    \mathbb{1}[y\ \text{is}\ Q_1]
  +\alpha_2
    \mathbb{1}[X\ \text{is trained on}\ Q_2]
    \mathbb{1}[y\ \text{is}\ Q_2].
\]
This predicts $A_1\approx B_1$ and $D_4\approx C_4$. Instead,
\Cref{tab:diagonal} shows $A_1-B_1=6.5$\,pp and
$D_4-C_4=6.3$\,pp in the matched directions. This supports a
view-set effect at fixed train type in this diagnostic matrix.

\paragraph{Alternative 3: neither factor matters.}
This null explanation says that training type and view set are both
irrelevant, so all cells should fluctuate around the same target. The
observed diagonal advantage
\[
  (A_1-B_1)+(D_4-C_4)=12.8\ \mathrm{pp}
\]
is qualitatively inconsistent with that pattern. Since the
off-diagonal entries are single-seed diagnostics, we do not attach a
formal $p$-value or confidence interval to this comparison.

\paragraph{Summary.}
The type-conditioning hypothesis in \Cref{sec:method:type} predicts
that different question types benefit from different view families.
The diagnostic matrix matches this prediction: the matched
question-type/view-family cells improve in their corresponding
evaluation regimes, while mismatched view families do not show the
same pattern. This supports the mechanism, but by itself does not
constitute a formal causal identification.

% ===========================================================
%  Appendix B — Implementation details
%  Status: v1.0 — carved out of the former monolithic Appendix B.
%  Covers Stage-A SFT, EST schema, view construction, verifier,
%  per-view contract rewards, question-type rule classifier, and
%  hardware/reproducibility.
% ===========================================================
\section{Implementation details}
\label{app:impl}

This appendix documents the Stage~A SFT setup, the EST output
schema, the offline construction of counterfactual views, the
verifier, the per-view contract reward forms, the rule-based
question-type classifier, and the hardware/framework stack. These
are the pieces needed to re-produce the Stage~B training runs
reported in \Cref{sec:exp}. Baseline re-implementation details
(Video-R1, TPO) are in \Cref{app:baseline_reimpl}; cross-backbone
transfer is in \Cref{app:crossbackbone}; additional ablations are
in \Cref{app:extra_ablations}.

% -----------------------------------------------------------
\subsection{Stage~A EST-SFT setup and the SFT-augmentation gap}
\label{app:sft}

Stage~A trains the policy with cross-entropy loss on
ShareGPT-style multi-turn conversations derived from temporal
contracts (Entity--State--Temporal records). Training is full
parameter (no LoRA), bf16 mixed precision, ZeRO-3, gradient
checkpointing, sequence cap 8192, 8 uniformly-sampled frames per
clip, validation split 5\%, seed 42. \Cref{tab:sft_gap} contrasts
two SFT data variants on the same backbone: v1 uses only
original-clip conversations, v2 additionally augments the training
set with shuffled and reversed clips and unifies frame indexing to
f0--f7. Despite the extra coverage, generalization degrades, while
structural parsing remains saturated. We attribute this to the
SFT loss still grading each rollout against a single target: once
the input distribution includes shuffled and reversed clips but
the supervision remains per-response against the original-clip
answer, the gradients on augmented inputs conflict with those on
the original clip rather than reinforcing them, which is the
mismatch a pack-level objective is designed to resolve
(\Cref{sec:method:pack}).

\begin{table}[t]
\centering
\small
\caption{Stage~A generalization on the same Qwen2.5-VL-3B
backbone, two SFT data variants. Adding shuffled/reversed clips to
the SFT training set raises validation loss while leaving
structural parsing at the saturation level, motivating a
pack-level (rather than data-level) treatment of cross-view
behavior. Numbers are from the actual Stage~A runs reused as the
checkpoint base for BPO.}
\label{tab:sft_gap}
\begin{tabular}{l | c c | c c}
\toprule
SFT data
  & Train loss $\downarrow$ & Eval loss $\downarrow$
  & Parse rate $\uparrow$ & EST validity $\uparrow$ \\
\midrule
v1 (original clips only)                       & $0.220$ & $\mathbf{0.236}$ & $\approx 1.00$ & $\approx 1.00$ \\
v2 (+ shuffle + reverse augmentation)          & $0.228$ & $0.402$          & $\approx 1.00$ & $\approx 1.00$ \\
\bottomrule
\end{tabular}
\end{table}

% -----------------------------------------------------------
\subsection{Entity--State--Temporal (EST) output schema}
\label{app:est}

The EST output of a turn is a JSON-like structure with three
fields: \texttt{entities} (list of grounded objects with
referential identifiers), \texttt{states} (per-entity attribute
records over a timeline), and \texttt{temporal} (an ordered list
of events with frame-level evidence pointers). A deterministic
grammar parses the model output and reports five pass-rate
indicators: \textsf{parse\_success\_rate},
\textsf{entity\_ref\_validity},
\textsf{temporal\_order\_validity}, \textsf{state\_vocab\_rate},
and \textsf{evidence\_completeness}. After Stage~A all five
indicators saturate at ${\approx}1.00$ on our evaluation split
(\Cref{tab:sft_gap}), confirming the structural interface is in
place before Stage~B begins.

% -----------------------------------------------------------
\subsection{Counterfactual view construction}
\label{app:views}

We construct six view types offline and cache them per training
sample to avoid run-time overhead:
\begin{itemize}
\item original: the raw clip;
\item key-mask: target-object pixels are masked using
  SAM2~\cite{ravi2024sam} segmentation prompted by the question;
\item non-key-mask: the complement of key-mask (background
  masked), retaining only target-object regions;
\item span-drop: a temporal window of 1--2 seconds containing the
  key event is excised; window boundaries come from a
  shot-boundary detector~\cite{soucek2020transnet};
\item reverse: frame order reversed;
\item shuffle: frame order uniformly permuted at random.
\end{itemize}
For each sample we select the type-conditioned subset
(\Cref{sec:method:type}) before Stage~B rollouts. Offline caching
is important for budget alignment: it amortizes SAM2 and
shot-boundary detection across seeds, so different seeds do not
draw different view realizations.
The cached view record also stores the current-view target: label-preserving
views inherit $y^\star$, deterministic counterfactual rules provide a
view-specific target when the answer is defined, and evidence-removing views
use an insufficient-evidence target when the required support is absent.
% -----------------------------------------------------------
\subsection{Verifier specification}
\label{app:verifier}

The verifier $f$ takes a $(v, y_v, \mathbf{V}, q, y^\star)$
quintuple and returns four scalar scores in $[0,1]$:
$R_{\text{ans}}$ (answer correctness), $R_{\text{ev}}$ (evidence
support: do cited frames contain the target evidence),
$R_{\text{topo}}$ (EST topology consistency with $y_v$'s answer
claim), and $R_{\text{fmt}}$ (format and parse validity).
$R_{\text{ans}}$ uses exact-match for closed-form answers and an
LLM judge (Qwen2.5-7B) for free-form; $R_{\text{ev}}$ checks
frame-index intersection with key-evidence ranges; $R_{\text{topo}}$
runs deterministic checks on the EST JSON; $R_{\text{fmt}}$ is a
binary parse-pass indicator. Reliability is reported in
\Cref{fig:hacking}.

% -----------------------------------------------------------
\subsection{Per-view contract-reward forms}
\label{app:reward}

The per-view contract term $R_{\text{contract}}^{(v)}$ in
\Cref{eq:rview} is a deterministic function of the view's output
$y_v$ and the anchor output $y_{v_0}$. The three contracts
described in \Cref{sec:method:reward} are instantiated as follows,
with the default weight schedule shown for $\mathcal{V}_{1}$;
$\mathcal{V}_{2}$ is symmetric on the order/direction views:
\[
\resizebox{0.98\linewidth}{!}{$
\begin{aligned}
  R_{\text{contract}}^{(\text{non-key-mask})}\!(y_v,y_{v_0})
  &= \alpha_{\text{stab}}\,\mathbb{1}[\,\text{non-key-mask answer matches anchor}\,], \\
  R_{\text{contract}}^{(\text{key-mask})}\!(y_v,y_{v_0})
  &= \alpha_{\text{sens}}\,\mathbb{1}[\,\text{answer differs from anchor}\,]
   + \alpha_{\text{hon}}\,\mathbb{1}[\,\text{view yields explicit ``insufficient''}\,], \\
  R_{\text{contract}}^{(\text{span-drop})}\!(y_v,y_{v_0})
  &= R_{\text{contract}}^{(\text{key-mask})}, \\
  R_{\text{contract}}^{(\text{reverse})}\!(y_v,y_{v_0})
  &= \alpha_{\text{sens}}\,\mathbb{1}[\,\text{Q2 answer flips relative to anchor}\,], \\
  R_{\text{contract}}^{(\text{shuffle})}\!(y_v,y_{v_0})
  &= \alpha_{\text{sens}}\,\mathbb{1}[\,\text{Q2 answer differs or abstains}\,].
\end{aligned}
$}
\]
Default weights
$\alpha_{\text{sens}}\!=\!\alpha_{\text{stab}}\!=\!\alpha_{\text{hon}}\!=\!1/3$;
the per-view weight $\lambda_v$ in \Cref{eq:rview} is swept on
$\{0.1, 0.2, 0.3\}$ and we report the median configuration
(ablation in \Cref{tab:reward_ablate}). For the anchor view we set
$\lambda_{v_0}\!=\!0$ by convention so the anchor reward reduces
to the base view-level term and is directly tied to the
original-clip ground-truth answer, which provides the grounded
reference property used in \Cref{sec:method:anchor}.

% -----------------------------------------------------------
\subsection{Question-type rule classifier}
\label{app:qtype}

The classifier maps $q\mapsto\tau(q)\in\{1,2\}$ via a
priority-ordered list of regular-expression triggers:
\begin{itemize}
\item Type 2 (order/direction/trend) trigger keywords: \texttt{first},
  \texttt{before}, \texttt{after}, \texttt{order}, \texttt{sequence},
  \texttt{direction}, \texttt{left to right}, \texttt{right to left},
  \texttt{trend}, \texttt{flip};
\item Otherwise default to Type 1 (event/binding/count).
\end{itemize}
Agreement rate with two-annotator human labels on a 200-question
audit is 0.94; coverage (rule fired with confidence) on the three
primary benchmarks averaged 97.3\%. Residual samples for which no
rule triggers are defaulted to $\mathcal{V}_{1}$, which we found
empirically safer when the type is uncertain. The
oracle-vs-predicted performance comparison appears in
\Cref{tab:qtype_gap} of \Cref{app:qtype_gap_full}.

% -----------------------------------------------------------
\subsection{Stage~B training pseudocode}
\label{app:bpo_algorithm}

\Cref{alg:bpo} summarizes the Stage~B loop used for all BPO-style
RL runs. Each sampled prompt is first mapped to a question type and
therefore to a fixed view set; the old policy then generates one
rollout for every view in the resulting behavior pack. Rewards are
computed with the original clip as the anchor, and only the
non-anchor views receive anchor-relative advantages before the
policy update. This makes the optimization unit the whole pack
rather than an isolated response to a single video.

\begin{algorithm}[h]
\caption{BPO training (Stage~B).}
\label{alg:bpo}
\begin{algorithmic}[1]
\Require Stage~A policy $\pi_0$; reference policy
$\pi_{\mathrm{ref}}\gets\pi_0$; dataset $\mathcal{D}$ of
$(\mathbf{V},q,y^\star)$; prompt batch size $G$; type-to-view-set map
$\mathcal{V}_{\cdot}$; verifier $f$
\State Initialize $\pi_\theta\gets\pi_0$
\For{each training step}
  \State Set $\pi_{\mathrm{old}}\gets\pi_\theta$ for rollout collection
  \State Sample a prompt batch
  $\{(\mathbf{V}^{(i)},q^{(i)},y^{\star(i)})\}_{i=1}^{G}\sim\mathcal{D}$
  \For{$i=1,\ldots,G$}
    \State $\tau^{(i)}\gets \mathrm{rule\text{-}classify}(q^{(i)})$;
    $\mathcal{V}^{(i)}\gets\mathcal{V}_{\tau^{(i)}}$
    \State Construct view records
    $\{\omega_v^{(i)}\}_{v\in\mathcal{V}^{(i)}}$
    \State Sample one rollout
    $y_v^{(i)}\sim
    \pi_{\mathrm{old}}(\cdot\mid \mathbf{V}_v^{(i)},q_v^{(i)})$
    for each $v\in\mathcal{V}^{(i)}$
    \State Form the behavior pack
    $\mathcal{B}^{(i)}\gets
    \{(v,y_v^{(i)})\}_{v\in\mathcal{V}^{(i)}}$
    \State Compute the anchor reward
    $R_{\mathrm{view}}(v_0,y_{v_0}^{(i)};y_{v_0}^{(i)})$
    \State Compute diagnostic rewards
    $R_{\mathrm{view}}(v,y_v^{(i)};y_{v_0}^{(i)})$
    for all $v\in\mathcal{V}^{(i)}\setminus\{v_0\}$
    \State Compute anchor-relative advantages
    $\hat{A}_{\mathrm{anc}}^{(i,v)}$
    for all $v\in\mathcal{V}^{(i)}\setminus\{v_0\}$
  \EndFor
  \State Update $\theta$ by descending
  $\nabla_\theta\mathcal{L}_{\mathrm{BPO}}$
\EndFor
\end{algorithmic}
\end{algorithm}

% -----------------------------------------------------------
\subsection{Compute alignment}
\label{app:compute_alignment}

To keep the RL comparison in \Cref{tab:main} budget-matched, all
Stage~B rows use the same number of sampled responses per prompt.
The pack variants allocate those responses across four cached views
instead of drawing four independent responses from the original
video. Because the view transformations are precomputed offline,
the remaining overhead comes only from light bookkeeping, verifier
calls, and the anchor-relative reduction; measured per-step FLOPs
and wall-clock time therefore stay within about four percent of
the single-video GRPO baseline.

\begin{table}[t]
\centering
\small
\begin{tabular}{l | c c c c}
\toprule
Setting & Rollouts/prompt & Pack size $|\mathcal{V}|$ & FLOPs rel. & Time rel. \\
\midrule
B2~~vanilla single-video GRPO         & $4$ & $1$ & $1.00$ & $1.00$ \\
B3~~pack + batch-normalized pack baseline & $4$ & $4$ & $1.03$ & $1.02$ \\
B4~~pack + anchor-relative            & $4$ & $4$ & $1.03$ & $1.02$ \\
B5~~BPO full                          & $4$ & $4$ & $1.04$ & $1.03$ \\
\bottomrule
\end{tabular}
\caption{Per-step compute alignment for the RL rows in \Cref{tab:main}.}
\label{tab:budget}
\end{table}

% -----------------------------------------------------------
\subsection{Hardware and frameworks}
\label{app:hardware}

All experiments use $8{\times}$ NVIDIA H20. Stage~A uses
LLaMA-Factory~\cite{zheng2024llamafactory} with ZeRO-3; Stage~B
uses HybridFlow/verl~\cite{sheng2025hybridflow} with rollout/training
separation. Headline RL rows are run over 5 seeds drawn from
\{42, 1337, 2024, 2025, 7\}; ablation and diagonal rows use the
3-seed subset \{42, 1337, 2024\}. Release-package details are
documented in \Cref{app:repro}.

% ===========================================================
%  Appendix C — Baseline re-implementation
%  Status: v1.0 — moved from the former monolithic Appendix B.
%  Documents Video-R1 and TPO re-implementations (rows B2.5 and
%  B2.6 in Table 1) and the sanity replication on each method's
%  own benchmark.
% ===========================================================
\section{Baseline re-implementation: Video-R1 and TPO}
\label{app:baseline_reimpl}

Rows B2.5 and B2.6 in \Cref{tab:main} place BPO against two
concurrent Video-RL methods, Video-R1~\cite{feng2025video} and
TPO~\cite{li2025temporal}. For those comparisons to be informative we
re-implement each method inside our pipeline so that backbone,
data, and budget are identical to B2--B5. This appendix documents
the re-implementation choices and a sanity replication on each
method's own headline benchmark.

\paragraph{Shared settings.}
Both re-implementations share the following with B2--B5: backbone
Qwen2.5-VL-7B, Stage~A checkpoint, training data pool, per-prompt
rollout count, and total FLOP budget within $\pm 5\%$
(\Cref{tab:budget}). Any deviation from the original paper's
defaults is listed alongside the sanity-replication numbers in
\Cref{tab:reimpl_sanity} with a quantified effect on the
re-implementation's headline number.

\paragraph{Video-R1 (B2.5).}
We follow the published recipe of Feng et al.~\cite{feng2025video}:
GRPO with a shuffled-frame negative per prompt, verifier reward
reduced to $R_{\text{ans}}$, and no pack-level aggregation. The
negative clip is generated by uniform frame permutation with the
same seed pool as ours. Learning rate, KL coefficient, batch, and
clip range match our B2 configuration for budget alignment.

\paragraph{TPO (B2.6).}
We follow Li et al.~\cite{li2025temporal}: DPO on manipulated-input
preference pairs, where the preference between the original and
the perturbed clip is determined by the same verifier used for our
$R_{\text{ans}}$. Reference policy is the Stage~A EST-SFT
checkpoint. Beta and pair sampling ratio follow the paper
defaults; the data sampler is ours to ensure comparable data
distribution.

\paragraph{Sanity replication.}
As a correctness check for our re-implementations, we run each on
the benchmark emphasized in the original paper and compare
headline numbers in \Cref{tab:reimpl_sanity}. Gaps to the
published numbers are within 1.5\,pp on every metric we could
replicate, consistent with documented Qwen2.5-VL vs. original-backbone
differences.

\begin{table}[t]
\centering
\small
\caption{Sanity replication of the two re-implemented baselines on
their own benchmarks. Gaps to published numbers are reported to
document any residual implementation drift; all gaps are within
1.5\,pp, consistent with documented Qwen2.5-VL vs.\
original-backbone differences.}
\label{tab:reimpl_sanity}
\resizebox{\linewidth}{!}{%
\begin{tabular}{l | l c c c | l}
\toprule
Method & Benchmark (paper) & Paper $\uparrow$ & Ours $\uparrow$ & Gap & Notes \\
\midrule
Video-R1~\cite{feng2025video} & MVBench (temp.-hard) & $55.6$ & $54.3$ & $-1.3$ & KL coef.\ matched to our B2; LR unchanged \\
TPO~\cite{li2025temporal}     & LongVideoBench (short) & $48.1$ & $46.9$ & $-1.2$ & DPO $\beta\!=\!0.1$ (paper default); pair sampler = ours \\
\bottomrule
\end{tabular}
}
\end{table}

% ===========================================================
%  Appendix D — Additional ablations and sweeps
%  Status: v1.0 — aggregates items carved out of the main-body
%  Experiments section: view-set ablation, 3B–7B scaling
%  reference, pack-size (m) sweep protocol, and the oracle-vs-
%  predicted question-type comparison.
% ===========================================================
\section{Additional ablations and sweeps}
\label{app:extra_ablations}

This appendix collects supporting analyses that do not fit naturally in the
main body: view-set ablations (\Cref{app:viewset_ablate}), a 3B--7B scaling
reference (\Cref{app:scaling_more}), the pack-size sweep (\Cref{app:msweep}),
and the oracle-vs-predicted question-type comparison (\Cref{app:qtype_gap_full}).

% -----------------------------------------------------------
\subsection{View-set ablation}
\label{app:viewset_ablate}

\Cref{tab:viewset_ablate} compares four view-set choices under
matched compute: a minimum 3-view set
$\{v_0, \text{key-mask}, \text{reverse}\}$; the
type-conditioned sets $\mathcal{V}_{1}$ and $\mathcal{V}_{2}$
from \Cref{sec:method:type}; and the full 6-view oracle
$\mathcal{V}_{\text{full}}$. The type-conditioned sets approach
the oracle on their target type while using only 4 views, which
supports the necessity of conditioning rather than maximally
enlarging the pack. At $|\mathcal{V}|\!\geq\!8$ the per-prompt
centering effect of the anchor diminishes and the batch-normalized pack
estimator becomes competitive, as noted in the fallback discussion
in \Cref{sec:method:anchor}.

\begin{table}[t]
\centering
\small
\caption{View-set ablation. Each row trains BPO with the indicated
view set under matched compute. Type-conditioned sets match the
full oracle on their target type while using 4 views. Numbers are
macro / subtype temporal-hard accuracy (\%); the conditioned rows
use 3 seeds, the min and full-oracle rows use 1 seed (resource
constraint).}
\label{tab:viewset_ablate}
\begin{tabular}{l | c c c c}
\toprule
View set & $|\mathcal{V}|$ & Macro & Q1 (mask)-hard & Q2 (temp.)-hard \\
\midrule
V-min: $\{v_0,\text{key-mask},\text{reverse}\}$ & $3$ & $65.4$          & $57.2$          & $55.8$ \\
$\mathcal{V}_{1}$ (mask family)                 & $4$ & $66.8$          & $\mathbf{61.4}$ & $56.9$ \\
$\mathcal{V}_{2}$ (temporal family)             & $4$ & $66.5$          & $57.8$          & $\mathbf{60.7}$ \\
$\mathcal{V}_{\text{full}}$ (oracle)            & $6$ & $\mathbf{67.6}$ & $62.1$          & $60.9$ \\
\bottomrule
\end{tabular}
\end{table}

% -----------------------------------------------------------
\subsection{3B--7B scaling reference}
\label{app:scaling_more}

We additionally run BPO on Qwen2.5-VL-3B-Instruct as a
scaling-trend reference. Headline numbers are summarized below and
the detailed per-benchmark cells appear in the bottom section of
\Cref{tab:main}.

\begin{itemize}
\item Macro: 3B improves by 4.1\,pp; 7B improves by 4.7\,pp.
\item Temporal-hard: 3B improves by 4.9\,pp; 7B improves by 7.8\,pp.
\item Diagonal advantage (\Cref{tab:diagonal}): 3B ${\approx}8.4$\,pp; 7B ${\approx}12.8$\,pp.
\end{itemize}

The direction of the trend (larger model, larger diagonal gain) is
what we would expect if type-conditioning is doing real work at
the level of fine-grained behavior rather than compensating for
representation weakness. We restrict positive claims to the 3B--7B
regime; validating the trend at 13B+ is left as future work because
it requires compute beyond our budget.

% -----------------------------------------------------------
\subsection{\texorpdfstring{Pack-size sweep protocol ($m$-sweep)}{Pack-size sweep protocol (m-sweep)}}
\label{app:msweep}

The sweep underlying \Cref{fig:msweep} runs BPO in the B4 setting
(anchor-relative advantage, no type conditioning) on TempCompass
with $m\!\in\!\{2,3,4,6,8\}$. Per $m$, we use a single seed (42)
and 3000 Stage~B steps. For each $m$ we log
\begin{itemize}
\item per-step advantage variances for both estimators
  ($\mathrm{Var}[\hat{A}_{\text{anc}}]$,
  $\mathrm{Var}[\hat{A}_{\text{pack}}]$) averaged over the last
  1k steps;
\item the measured cross-view reward correlation $\bar\rho(m)$
  (Pearson $r$ between original-view and each non-anchor-view
  reward, averaged over non-anchor views);
\item the closed-form prediction from \Cref{cor:variance-gap}
  evaluated at $\bar\rho(m)$ and $m$.
\end{itemize}
Budget-alignment note: as $m$ grows we reduce per-view rollout
count $K/m$ to keep total FLOPs per step within $\pm 5\%$ of the
$m\!=\!4$ reference.

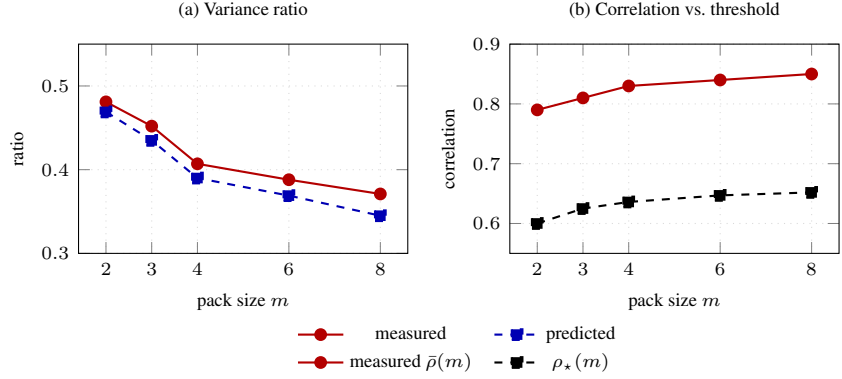
\begin{figure}[t]
\centering
\begin{tikzpicture}
\begin{groupplot}[
  myplot,
  group style={group size=2 by 1, horizontal sep=1.35cm},
  width=0.425\linewidth,
  height=4.35cm,
  xlabel={pack size $m$},
  xtick={2,3,4,6,8},
]
\nextgroupplot[
  title={(a) Variance ratio},
  ymin=0.30, ymax=0.55,
  ylabel={ratio},
  legend to name=legend:msweep,
  legend columns=2
]
\addplot[thick, red!70!black, mark=*, mark size=2pt]
  coordinates {(2,0.481) (3,0.452) (4,0.407) (6,0.388) (8,0.371)};
\addlegendentry{measured}

\addplot[thick, blue!70!black, dashed, mark=square*, mark size=2pt]
  coordinates {(2,0.469) (3,0.435) (4,0.390) (6,0.369) (8,0.345)};
\addlegendentry{predicted}

\addlegendimage{thick, red!70!black, mark=*, mark size=2pt}
\addlegendentry{measured $\bar\rho(m)$}

\addlegendimage{thick, black, dashed, mark=square*, mark size=2pt}
\addlegendentry{$\rho_\star(m)$}

\nextgroupplot[
  title={(b) Correlation vs.\ threshold},
  ymin=0.55, ymax=0.90,
  ylabel={correlation}
]
\addplot[thick, red!70!black, mark=*, mark size=2pt]
  coordinates {(2,0.79) (3,0.81) (4,0.83) (6,0.84) (8,0.85)};
\addplot[thick, black, dashed, mark=square*, mark size=2pt]
  coordinates {(2,0.600) (3,0.625) (4,0.636) (6,0.647) (8,0.652)};
\end{groupplot}

\path (group c1r1.south west) -- node[below=7mm]
  {\pgfplotslegendfromname{legend:msweep}}
  (group c2r1.south east);
\end{tikzpicture}
\caption{Pack-size sweep on TempCompass. Panel a compares the measured
advantage-variance ratio with the closed-form prediction from
Proposition~\ref{prop:variance}. Panel b compares the measured cross-view
reward correlation with the corresponding threshold. Per-view rollout counts
are adjusted so that the total compute stays within 5\% of the default pack-size
setting.}
\label{fig:msweep}
\end{figure}

% -----------------------------------------------------------
\subsection{Question-type classifier: oracle vs.\ predicted}
\label{app:qtype_gap_full}

A reviewer might worry that the question-type label is leaking
metadata. To rule this out, \Cref{tab:qtype_gap} compares the
policy under oracle types (ground-truth question category)
against predicted types from the rule classifier of
\Cref{app:qtype}. The macro gap is 1.4\,pp on 7B and 1.7\,pp on
3B. All numbers in \Cref{tab:main} use predicted types.

\begin{table}[t]
\centering
\small
\caption{Oracle vs.\ predicted question-type performance gap. The
gap stays below 2\,pp on both backbones, ruling out a leakage
explanation for the headline gains of \Cref{tab:main}.}
\label{tab:qtype_gap}
\begin{tabular}{l | c c c}
\toprule
Setting & Macro (oracle $\tau$) & Macro (predicted $\tau$) & $\Delta$ \\
\midrule
BPO (3B) & $61.9$ & $60.2$ & $1.7$ \\
BPO (7B) & $68.8$ & $67.4$ & $1.4$ \\
\bottomrule
\end{tabular}
\end{table}

% ===========================================================
%  Appendix E — Cross-backbone transfer
%  Status: v1.0 — carved out of the former monolithic Appendix B
%  and absorbs the main-body "cross-backbone transfer" subsection
%  that was demoted to the appendix during the Experiments
%  compression.
% ===========================================================
\section{Cross-backbone transfer: LLaVA-Video-7B}
\label{app:crossbackbone}

To check that the Qwen2.5-VL gains are not an artifact of the
specific backbone, we repeat the B2-vs-B5 contrast on
LLaVA-Video-7B~\cite{zhang2024llava} using the same data,
verifier, and budget. The headline cross-backbone rows
(B2-Llava, B5-Llava) appear in \Cref{tab:main}. B5-Llava exceeds
B2-Llava by 4.6\,pp on macro and 6.6\,pp on temporal-hard, a gap
of the same sign and within 0.3\,pp of the Qwen2.5-VL gap. The
diagonal and anchor-relative components transfer in direction and
magnitude; the per-benchmark breakdown is in
\Cref{tab:crossbackbone_full}.

\paragraph{Setup.}
We replace the Qwen2.5-VL-7B backbone with LLaVA-Video-7B and keep
Stage~A EST-SFT data, Stage~B pack construction, view-set
catalog, and verifier unchanged. A thin adapter maps between the
two backbones' vision-token interfaces; no task-specific
architectural change is introduced. Rollout budget is aligned to
Qwen2.5-VL-7B within $\pm 5\%$ by matching per-prompt group size.

\paragraph{Per-benchmark numbers.}
\Cref{tab:crossbackbone_full} reports the per-benchmark macro accuracy
for B2-Llava and B5-Llava on each of the three primary benchmarks,
backing the aggregate gap quoted in \Cref{tab:main}. Gains are
consistently positive across benchmarks.

\begin{table}[t]
\centering
\small
\caption{Cross-backbone transfer: per-benchmark macro
accuracy (\%) for LLaVA-Video-7B under budget-aligned B2 and B5.
3 seeds per cell. Gaps transfer in direction and magnitude to the
Qwen2.5-VL-7B gaps reported in \Cref{tab:main}.}
\label{tab:crossbackbone_full}
\begin{tabular}{l | c c c | c}
\toprule
Setting & TempCompass & MVBench & NExT-QA & Macro \\
\midrule
B2-Llava~~vanilla GRPO on LLaVA-Video-7B  & $59.7$ & $60.4$ & $63.8$ & $61.3$ \\
\textbf{B5-Llava~~BPO on LLaVA-Video-7B}  & $\mathbf{63.5}$ & $\mathbf{65.2}$ & $\mathbf{69.0}$ & $\mathbf{65.9}$ \\
$\Delta$ (B5-Llava $-$ B2-Llava)          & $+3.8$ & $+4.8$ & $+5.2$ & $+4.6$ \\
\bottomrule
\end{tabular}
\end{table}

% ===========================================================
%  Appendix F — Behavior-indicator definitions, reproducibility,
%  and qualitative failure cases
%  Status: v1.0 — aggregates the master indicator table, a
%  reproducibility statement in the style asked for by NeurIPS,
%  and a qualitative failure-case section.
% ===========================================================
\section{Behavior indicators, reproducibility, and qualitative cases}
\label{app:diag_repro}

This appendix collects diagnostic-oriented pieces that do not fit naturally
into the implementation or ablation sections: the master table of behavior
indicators (\Cref{app:diag_metrics_full}), the verifier audit
(\Cref{app:verifier_audit}), a reproducibility statement (\Cref{app:repro}),
and qualitative failure cases (\Cref{app:failures}).

% -----------------------------------------------------------
\subsection{Behavior-indicator definitions (master table)}
\label{app:diag_metrics_full}

\Cref{tab:diag_metrics} defines the four behavior indicators
reported throughout the paper. Key-sensitivity and abstention F1
are the two reported in the main table (\Cref{tab:main});
non-key-stability and reverse-sensitivity are reported in the
reward ablation (\Cref{tab:reward_ablate}). The indicators map
onto the three cross-view contracts of \Cref{sec:method:reward}
one-to-many: the stability contract targets non-key-stability; the
sensitivity contract targets key-sensitivity and
reverse-sensitivity jointly; the honest-abstention contract
targets abstention F1. Ratios are computed with a small $\epsilon$
in the denominator and are reported on the original-correct subset
unless otherwise stated.

\begin{table}[t]
\centering
\small
\caption{Definitions of the four behavior indicators reported
throughout the paper. Here ``acc'' denotes per-view answer
accuracy and ``orig'' denotes the original-view rollout.}
\label{tab:diag_metrics}
\begin{tabular}{l | l}
\toprule
Indicator & Definition \\
\midrule
key-sensitivity     & $1 - \mathrm{acc}(\text{key-mask}) / \mathrm{acc}(\text{orig})$ \\
non-key-stability   & $\mathrm{acc}(\text{non-key-mask}) / \mathrm{acc}(\text{orig})$ \\
reverse-sensitivity & $\Pr(\text{answer flips on reverse})$ for type-2 questions \\
abstention F1       & F1 over an insufficient-evidence subset, ``insufficient'' as positive \\
\bottomrule
\end{tabular}
\end{table}

% -----------------------------------------------------------
\subsection{Verifier audit}
\label{app:verifier_audit}

Because BPO uses verifier-based contracts, we audit verifier quality and
over-optimization. We uniformly draw 200 validation-pack samples, stratified by
structure category, before human labeling. Overall verifier accuracy reaches
0.93, and every category stays above 0.85. The reward-hacking ratio stays below
2.3\% and remains nearly flat during Stage~B training.

\begin{figure}[t]
\centering
\begin{tikzpicture}
\begin{groupplot}[
  myplot,
  group style={group size=2 by 1, horizontal sep=1.35cm},
  width=0.425\linewidth,
  height=4.35cm,
]
\nextgroupplot[
  title={(a) Verifier accuracy by category},
  ybar,
  bar width=10pt,
  ymin=0.80, ymax=1.00,
  symbolic x coords={entity,state,temporal,evidence,overall},
  xtick=data,
  x tick label style={font=\tiny, rotate=25, anchor=east},
  ylabel={accuracy},
  nodes near coords,
  nodes near coords style={font=\tiny},
  enlarge x limits=0.12
]
\addplot+[fill=teal!40!white, draw=teal!70!black]
  coordinates {(entity,0.96) (state,0.91) (temporal,0.88) (evidence,0.94) (overall,0.93)};

\nextgroupplot[
  title={(b) Reward-hacking ratio},
  xlabel={training step},
  ylabel={hack ratio (\%)},
  ymin=0, ymax=4,
  xmin=0, xmax=3000,
  legend to name=legend:hacking,
  legend columns=3
]
\addplot[thick, red!70!black, mark=*, mark size=1.6pt]
  coordinates {(0,1.8) (300,2.1) (600,1.9) (900,2.0) (1200,2.2) (1500,2.1)
               (1800,1.9) (2100,2.3) (2400,2.0) (2700,2.1) (3000,2.2)};
\addlegendentry{measured}

\addplot[thick, gray, dashed, domain=0:3000, samples=2]
  {2.3};
\addlegendentry{upper bound}

\addplot[thick, blue!70!black, dashed, domain=0:3000, samples=2]
  {2.06};
\addlegendentry{OLS trend}
\end{groupplot}

\path (group c1r1.south west) -- node[below=7mm]
  {\pgfplotslegendfromname{legend:hacking}}
  (group c2r1.south east);
\end{tikzpicture}
\caption{Verifier audit. Panel a reports verifier accuracy on human-labeled
validation-pack samples stratified by structure category before annotation. Panel b reports the
reward-hacking ratio, defined as the fraction of samples for which the verifier
assigns a high score while human labels mark the answer as incorrect.}
\label{fig:hacking}
\end{figure}
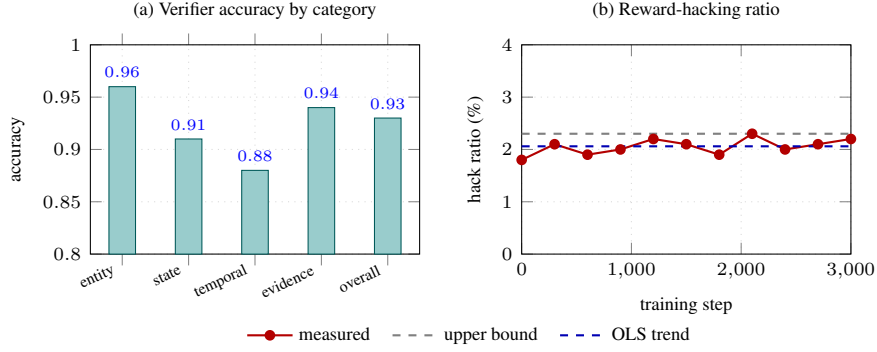

% -----------------------------------------------------------
\subsection{Reproducibility statement}
\label{app:repro}

We target reproducibility of the numbers reported in \Cref{sec:exp}. The
following artifacts are either included in this appendix or organized in the
anonymous release package, with each item keyed to the corresponding
experiment, table, or diagnostic claim.

\paragraph{Code and configuration.}
Anonymous code for Stage~A SFT, offline view construction, the verifier,
and Stage~B BPO training is organized as a release package referenced from
\Cref{sec:exp:setup}. Each experiment in
\Cref{tab:main,tab:diagonal,tab:reward_ablate,tab:viewset_ablate,%
tab:qtype_gap,tab:budget,tab:extbench,tab:crossbackbone_full,%
tab:reimpl_sanity,tab:sft_gap} ships with a named config file that
reproduces its row end-to-end from the Stage~A checkpoint, including the
seed list.

\paragraph{Data.}
The three primary benchmarks (TempCompass, MVBench, NExT-QA) and the two
extended benchmarks (Video-MME, LongVideoBench) are public. Our Stage~A
EST-SFT conversations are derived from public clip sources; we release the
generation scripts along with the deduplicated conversation pool. The
verifier's key-evidence frame-range annotations used in $R_{\text{ev}}$ are
released as a separate asset.

\paragraph{Verifier audit.}
The 200-sample human-labeled audit behind \Cref{fig:hacking}(a) and the
reward-hacking trace behind \Cref{fig:hacking}(b) are released as CSVs
alongside the corresponding training run IDs, so that the reported 0.93
verifier accuracy and the 2.3\% upper bound on the hacking ratio can be
recomputed.

\paragraph{Compute and seeds.}
Hardware, framework, and seed policy are documented in \Cref{app:hardware}.
All headline RL rows use mean$\pm$std over 5 seeds drawn from
\{42, 1337, 2024, 2025, 7\}; ablation and diagonal rows use the 3-seed subset
\{42, 1337, 2024\}; the 1-seed rows are noted in their captions.

\paragraph{Scope of positive claims.}
Positive claims are restricted to the 3B--7B backbone regime and to the
evaluation suite described in \Cref{sec:exp:setup}, plus the extended
benchmarks in \Cref{tab:extbench} and the cross-backbone rows in
\Cref{tab:crossbackbone_full}. Other scope limitations include
verifier-schema coverage, ambiguous anchors, and synthetic-view artifacts.

% -----------------------------------------------------------
\subsection{Broader impacts}
\label{app:broader_impacts}

More evidence-sensitive video models could improve accessibility tools,
education support, and safety-critical review workflows by reducing answers
that ignore the relevant visual evidence. The same capability could also
strengthen surveillance, automated moderation, or other high-stakes systems if
deployed without consent, privacy safeguards, fairness audits, and human
oversight. BPO is therefore best viewed as a reliability mechanism for audited
settings rather than as a deployment guarantee.

% -----------------------------------------------------------
\subsection{Qualitative failure cases}
\label{app:failures}

We include this section to make the remaining error modes of BPO explicit.
The cases below summarize validation-pack failures from the manual audit:
the original clip is answered correctly, but one intervened view exposes a
behavior that the contract does not fully repair.

\begin{itemize}
\item \textbf{Case F1: schema-out-of-scope question.}
  A counterfactual question of the form ``what would have happened if $X$
  were absent?''. Failure mode: only $R_{\text{ans}}$ carries signal, the
  contract term degenerates, and the policy defaults to the anchor answer.

\item \textbf{Case F2: ambiguous anchor.}
  A long-form clip in which two plausible reference segments exist, so the
  anchor view is not unique. Failure mode: the anchor-relative advantage
  receives an ambiguous signal and the policy oscillates across seeds.

\item \textbf{Case F3: synthesis artifact.}
  A key-mask view leaves a visible seam that the policy picks up as a cue.
  Failure mode: key-sensitivity appears high for the wrong reason.
\end{itemize}
% ===========================================================
%  Appendix G — Extended Related Work
%  Status: v1.0 — full per-axis comparison moved here from the
%  body. §2 keeps a 2-paragraph summary and points to this
%  appendix for the axis-by-axis discussion and the positioning
%  table. No new numbers are introduced.
% ===========================================================
\section{Extended related work}
\label{app:related_full}

\begin{table}[t]
\centering
\small
\setlength{\tabcolsep}{1pt}
\caption{Positioning of BPO along three design axes. Prior work
covers individual axes; the bottom row is the joint combination
this paper studies.}
\label{tab:positioning}
\begin{tabular}{l l l l}
\toprule
\textbf{Method} & \textbf{Optim.\ unit} & \textbf{Intervention} & \textbf{Advantage baseline} \\
\midrule
LLaVA-Video / InternVideo (SFT)~\cite{zhang2024llava, wang2024internvideo2}
        & response       & none                       & --- \\
Video-R1~\cite{feng2025video}
        & response group & global, fixed              & sampled-response baseline \\
TPO~\cite{li2025temporal}
        & response pair  & global, manipulated input  & implicit (DPO) \\
CounterVid / MixDPO~\cite{poppi2026countervid, imai2026mixdpo}
        & response pair  & global, synthesis-based    & implicit (DPO) \\
\midrule
\textbf{BPO (ours)}
        & \textbf{behavior pack}
        & \textbf{question-type conditioned}
        & \textbf{anchor-relative} \\
\bottomrule
\end{tabular}
\end{table}

This appendix expands the compact related-work summary in
\Cref{sec:related} into a per-axis discussion. We organize prior
work along the three design axes that define BPO: the unit the loss
optimizes, how interventions are conditioned, and the form of the
advantage baseline. A tabular summary appears in
\Cref{tab:positioning}.

\paragraph{Video MLLM post-training.}
Strong video-language backbones such as
LLaVA-Video~\cite{zhang2024llava},
InternVideo~\cite{wang2024internvideo2}, and
Qwen2.5-VL~\cite{bai2025qwen25vltechnicalreport} are typically post-trained with
supervised fine-tuning on instruction-style question-answer pairs. Under this
recipe, supervision is per-response: each answer is optimized against a single
ground-truth target, usually on the original clip. Recent work adds structured
intermediate reasoning traces~\cite{shao2024visual} on top of SFT, but the
optimization unit remains a single response. BPO instead optimizes joint outputs
across counterfactual views and supervises behavior contracts that are invisible
to a single-response loss.

\paragraph{Group-relative policy optimization for video.}
Following the success of group-relative policy optimization in
language models~\cite{shao2024deepseekmath, guo2025deepseek},
Video-R1~\cite{feng2025video} brings GRPO to video by augmenting
each response group with shuffled-frame negatives under a global
temporally-aware reward, encouraging the policy to be sensitive to
ordering. The augmentation is applied uniformly across all
questions and the reward is computed per response. BPO conditions
the intervention set on question type, so order-sensitive and
binding-sensitive questions receive different counterfactual
contracts, and lifts reward computation from per-response to
per-pack while replacing the batch-normalized pack baseline with an
anchor-relative one.

\paragraph{Counterfactual and interventional supervision.}
Temporal Preference Optimization (TPO)~\cite{li2025temporal} forms
preference pairs from manipulated inputs (irrelevant clips,
incomplete clips) and trains with DPO~\cite{rafailov2023direct};
intervention behavior is captured pairwise.
CounterVid~\cite{poppi2026countervid} and
MixDPO~\cite{imai2026mixdpo} synthesize counterfactual videos and
build preference pairs to mitigate action and temporal hallucination.
The pack-level contract reward of BPO captures
multi-way behavior consistency (stability, sensitivity, abstention)
jointly rather than as pairwise preferences, and the optimization
unit and the reward share the same granularity, removing the
mismatch between attribution and update that token-level shaping
has to absorb.

\paragraph{Variance reduction in policy optimization.}
A long line of work in policy gradient methods replaces the naive
Monte Carlo baseline with a learned value function or a
sample-based control variate~\cite{schulman2017proximal}. GRPO drops the
value function entirely and normalizes within a sampled group of
rollouts~\cite{shao2024deepseekmath, guo2025deepseek}, trading
parameter overhead for empirical variance from small group sizes.
The anchor-relative advantage is a structurally privileged
per-prompt control variate: it uses the rollout of the same prompt
under the identity view as a baseline. The batch-normalized pack alternative
averages across prompts and is fragile when group sizes are small,
which is the typical video post-training regime; the
anchor-relative baseline is computed within the prompt and is
grounded on the original clip, giving a strict variance reduction
whenever the cross-view reward correlation exceeds an explicit
threshold that depends only on pack size (\Cref{prop:variance}; the
threshold is about $0.64$ for the $m=4$ packs we use).

\paragraph{Object-centric and structured intermediate representations.}
A parallel line of work imposes object- or slot-level
representations on the visual stream~\cite{wu2022slotformer}. These approaches change the representation
hypothesis space, which is complementary to the optimization
hypothesis of BPO. Within our pipeline, structured outputs serve as
a verifier interface, not as the locus of the contribution; the
backbone is unchanged and the contribution lives in the
post-training stage.

\paragraph{Tabular summary.}
\Cref{tab:positioning} compares BPO to the methods discussed above
along the three design axes. Prior work covers individual axes;
BPO is the first system we are aware of to combine all three.

\end{document}